\documentclass[preprint,12pt]{elsarticle}

\usepackage{amssymb}
\rmfamily{\fontsize{12pt}{\baselineskip}\selectfont}

\usepackage{graphicx}
\usepackage{amsfonts, amsmath, amssymb, amsbsy, amsthm}
\usepackage{bm}
\usepackage{color}
\usepackage{xcolor}
\usepackage{booktabs}
\usepackage{tikz}
\usepackage{caption}  
\usetikzlibrary{shapes.geometric, arrows}
\tikzstyle{block} = [rectangle, draw, fill=gray!20, text width=7em, text centered, rounded corners, minimum height=2em, font=\small]
\tikzstyle{line} = [draw, -latex']
\tikzstyle{dashedline} = [draw, -latex', dashed]  % 定义虚线样式
\usepackage{hyperref}
\usepackage{mathrsfs}
\usepackage{longtable}
\usepackage[top=2.5cm,bottom=2.0cm,left=2.0cm,right=2.0cm]{geometry}
\usepackage{epstopdf}
\usepackage{stmaryrd}
\usepackage[normalem]{ulem}
\usepackage{cancel}
\usepackage{scalerel}
\usepackage{enumitem}
\usepackage{algorithm}
\usepackage{subfigure}
\usepackage{algorithm, algorithmicx}
\usepackage[noend]{algpseudocode}
\usepackage{lineno}
\usepackage{enumitem}
\usepackage{diagbox}
\usepackage{multirow}
\usepackage{siunitx}
\usepackage{booktabs}
\renewcommand{\algorithmiccomment}[1]{\bgroup\hfill//~#1\egroup}
\usepackage{algorithm}
\usepackage{algpseudocode}
\usepackage{makecell} 
\usepackage{amsmath}

\newtheorem{theorem}{Theorem}[section]

\usepackage{amsthm} % ensure you have this
\newtheorem{remark}{Remark}[section]

\begin{document}

\begin{frontmatter}

\title{A Conformal Prediction Framework for Uncertainty Quantification in Physics-Informed Neural Networks}

\author[nusaddress]{Yifan Yu}

\author[ubcaddress]{Cheuk Hin Ho}

\author[nusaddress]{Yangshuai Wang}
\ead{yswang@nus.edu.sg}

\address[nusaddress]{Department of Mathematics, National University of Singapore, 10 Lower Kent Ridge Road, 119076, Singapore.}

\address[ubcaddress]{Department of Mathematics, University of British Columbia, Vancouver, V6T1Z2, Canada.}

\begin{abstract}
Physics-Informed Neural Networks (PINNs) have emerged as a powerful framework for solving PDEs, yet existing uncertainty quantification (UQ) approaches for PINNs generally lack rigorous statistical guarantees. In this work, we bridge this gap by introducing a distribution-free conformal prediction (CP) framework for UQ in PINNs. This framework calibrates prediction intervals by constructing nonconformity scores on a calibration set, thereby yielding distribution-free uncertainty estimates with rigorous finite-sample coverage guarantees for PINNs. To handle spatial heteroskedasticity, we further introduce local conformal quantile estimation, enabling spatially adaptive uncertainty bands while preserving theoretical guarantee. Through systematic evaluations on typical PDEs (damped harmonic oscillator, Poisson, Allen–Cahn, and Helmholtz equations) and comprehensive testing across multiple uncertainty metrics, our results demonstrate that the proposed framework achieves reliable calibration and locally adaptive uncertainty intervals, consistently outperforming heuristic UQ approaches. By bridging PINNs with distribution-free UQ, this work introduces a general framework that not only enhances calibration and reliability, but also opens new avenues for uncertainty-aware modeling of complex PDE systems.
\end{abstract}

%%%Graphical abstract
%\begin{graphicalabstract}
%%\includegraphics{grabs}
%\end{graphicalabstract}

% %%Research highlights
% \begin{highlights}
% \item Rigorous framework to explain generalisation of machine-learned interatomic potentials
% \item Quantify prediction error in terms of training data 
% \item Towards rigorous MLIPs workflow for materials defects
% \item Numerical experiments validate and refine MLIPs best practices
% \end{highlights}

% \begin{keyword}
% %% keywords here, in the form: keyword \sep keyword
% machine-learned interatomic potentials \sep foundation models \sep fine-tuning \sep benchmark \sep molecular 
% %% PACS codes here, in the form: \PACS code \sep code
% %% MSC codes here, in the form: \MSC code \sep code
% %% or \MSC[2008] code \sep code (2000 is the default)

% \end{keyword}

\end{frontmatter}

\section{Introduction}
\label{sec:intro}

% PINN
Physics-Informed Neural Networks (PINNs) have emerged as a versatile framework for solving partial differential equations (PDEs) by embedding physical laws into neural network training~\cite{raissi_physics-informed_2019, karniadakis2021physics}. Numerous variants have been developed to enhance accuracy, efficiency, and applicability~\cite{mcclenny2023self, liao2021deep, yu2018deep, yu2022gradient, toscano2025pinns, zang2020weak}, enabling PINNs to address complex geometries~\cite{xiang2022hybrid, costabal2024delta}, high-dimensional and multiscale problems~\cite{hu2025bias, wang2021eigenvector, li2020multi}, and inverse formulations~\cite{lu2021physics, chen2020physics} within a unified mesh-free paradigm. Applications span fluid mechanics~\cite{cai_physics-informed_fluid_2021, wessels2020neural}, heat transfer~\cite{cai_physics-informed_heat_2021-1, jalili2024physics}, and materials science~\cite{misyris_physics-informed_2020, zhang2022analyses}; see~\cite{cai_physics-informed_fluid_2021, cuomo2022scientific, de2024numerical, zhao2024comprehensive, shukla2024comprehensive} for comprehensive reviews. Despite this progress, existing PINNs are almost exclusively deterministic, lacking a principled mechanism for quantifying predictive uncertainty—an essential capability for reliable scientific computing.

Uncertainty quantification (UQ) is essential for reliable scientific computing. In PINNs, predictive uncertainty arises from data scarcity, model misspecification, and non-convex optimization, and is further compounded by the absence of a probabilistic formulation and the high-dimensional solution space. Existing strategies, such as dropout approximations, ensembles, stochastic gradient perturbations, and Bayesian PINNs (via Hamiltonian Monte Carlo or variational inference), introduce randomness to capture epistemic uncertainty~\cite{yang_b-pinns_2021, gal_dropout_2016, zou_uncertainty_2025, haitsiukevich_improved_2023, alhajeri_physics-informed_2022}. However, these methods rely on strong distributional assumptions and often lack rigorous coverage guarantees, limiting their reliability. This motivates the development of distribution-free, statistically principled frameworks that enable uncertainty estimation without explicit probabilistic modeling.

Conformal prediction (CP) is a statistically principled framework for uncertainty quantification that provides distribution-free prediction intervals with guaranteed coverage under minimal assumptions such as data exchangeability~\cite{shafer2008tutorial, angelopoulos2023conformal}. In contrast to the aforementioned UQ approaches, CP is a distribution-free post hoc wrapper: it requires no access to model internals and can be combined with any baseline uncertainty estimator to construct prediction intervals with user-specified significance levels and guaranteed finite-sample coverage. Extensions such as conditional coverage–guaranteed CP~\cite{gibbs_conformal_2024} have been proposed to enhance flexibility, though their applicability to PDE-based UQ remains unexplored.  

Recently, CP has attracted increasing attention in scientific machine learning due to its theoretical guarantees and computational efficiency. For example, Hu et al.~\cite{hu2022robust} combined CP with latent-space distance metrics to calibrate uncertainty for interatomic potentials, Moya et al. incorporated CP into Deep Operator Networks~\cite{moya_conformalized-deeponet_2025, Moya2023DeepONetGridUQ} and Kolmogorov-Arnold Networks~\cite{moya2025conformalized}, achieving finite-sample coverage guarantees for operator learning and function approximating tasks. Gopakumar et al.~\cite{gopakumar_uncertainty_2024} further demonstrated CP-based UQ in surrogate modeling for spatio-temporal systems, including PDE solvers and weather forecasting. These studies underscore the growing promise of CP as a general-purpose tool for UQ in scientific modeling. To the best of our knowledge, no prior work has integrated CP into PINNs or systematically evaluated its performance against heuristic UQ methods, leaving an important gap that motivates the present study.

In this work, we introduce a CP–based framework for UQ in PINNs. The method is distribution-free and guarantees finite-sample coverage, thereby addressing a central limitation of deterministic PINNs. Prediction intervals are constructed from nonconformity scores evaluated on a calibration set, requiring only minimal assumptions and leaving existing PINN architectures and training pipelines unchanged. We investigate three representative UQ methods: distance-based UQ, Monte Carlo droupout, and Bayesian posterior sampling. Their predictive variances are calibrated using conformal prediction as a post hoc tool. Calibration quality is assessed through empirical coverage and average coverage deviation. Beyond standard CP, we develop a localized conformal quantile estimation strategy that adapts prediction intervals to heteroskedastic regimes, yielding sharper yet statistically robust uncertainty bands that faithfully reflect spatial variability in PDE solutions. 

Extensive experiments on canonical PDE benchmarks (damped harmonic oscillation, Poisson, Allen–Cahn, and Helmholtz equations) demonstrate that the proposed framework consistently yields reliable and well-calibrated uncertainty estimates. The local CP, in particular, accurately identifies regions of elevated uncertainty while maintaining sharper intervals than standard CP. The proposed framework not only advances the theoretical foundations of UQ for PDE solvers, but also provides a practical, extensible methodology for uncertainty-aware scientific computing.

\paragraph{Outline} This paper is organized as follows. Section~\ref{sec:pinn} reviews the formulation of physics-informed neural networks.
Section~\ref{sec:cp&uq} presents the conformal prediction framework for PINNs, together with the heuristic UQ baselines used for comparison. Section~\ref{sec:numerics} reports numerical experiments on benchmark PDEs, with detailed evaluations of uncertainty quantification performance. Section~\ref{sec:extension} discusses extensions of the method based on localized conformal prediction, presenting an algorithm with rigorous coverage guarantees together with numerical validation. 
Finally, Section~\ref{sec:conclusion} concludes the paper and outlines future research directions.

\section{Background: Physics-Informed Neural Networks (PINNs)}
\label{sec:pinn}

In this section, we provide a brief overview of PINNs. In Section~\ref{sec:sub:basics}, we introduce the fundamental formulation of PINNs, which serves as the foundation for our proposed method. In Section~\ref{sec:sub:analysis}, we discuss existing analytical perspectives on PINNs and clarify their connection to uncertainty quantification.

\subsection{Basic Framework}
\label{sec:sub:basics}

Let $\Omega \subset \mathbb{R}^d$ be a bounded spatial domain and $T > 0$ a terminal time. We consider the generic initial-boundary value problem for a solution field $u: \Omega \times [0, T] \to \mathbb{R}^{n}$:
\begin{align}
\mathcal{L}[u](\mathbf{x},t) &= f(\mathbf{x},t), &&\quad (\mathbf{x},t) \in \Omega \times (0,T], \label{eq:pde}\\
u(\mathbf{x},0) &= u_0(\mathbf{x}), &&\quad \mathbf{x} \in \Omega, \label{eq:ic} \\
\mathcal{B}[u](\mathbf{x},t) &= g(\mathbf{x},t), &&\quad (\mathbf{x},t) \in \partial\Omega \times (0,T], \label{eq:bc}
\end{align}
where $\mathcal{L}[\cdot]$ is a (possibly nonlinear) differential operator acting on $u$, $f$ is a known source term, and $\mathcal{B}[\cdot]$ denotes a boundary operator, such as Dirichlet or Neumann conditions. The functions $u_0$ and $g$ specify the initial and boundary data, respectively.

PINNs aim to approximate the solution $u$ using a neural network $u_\theta : \Omega \times [0,T] \to \mathbb{R}^{n}$, parameterized by $\theta \in \mathbb{R}^{d_\theta}$. The surrogate $u_\theta$ is trained to simultaneously satisfy the governing PDE~\eqref{eq:pde}, along with its associated initial and boundary conditions~\eqref{eq:ic}–\eqref{eq:bc}, and to fit any available observational data. This approach enables a seamless integration of data and physics, where the corresponding loss components are often treated in a multi-objective optimization framework~\cite{rohrhofer2023data}.

Suppose we are given a set of observation data $\mathcal{D}_{\mathrm{data}} := \big\{(x^{(i)}, u^{(i)})\big\}_{i=1}^{N_{\rm d}} := \big\{ (\mathbf{x}^{(i)}, t^{(i)}, u^{(i)}) \big\}_{i=1}^{N_{\rm d}}$ collected at discrete sensor locations. These data represent noisy measurements of the true solution $u$. To enforce physical consistency, we introduce three additional point sets: $\mathcal{D}_{\mathrm r} = \big\{(\mathbf{x}^{(j)}_{\mathrm r}, t^{(j)}_{\mathrm r})\big\}_{j=1}^{N_{\mathrm r}} \subset \Omega \times (0,T]$ for the PDE residual, $\mathcal{D}_{\mathrm i} = \big\{(\mathbf{x}^{(l)}_{\mathrm i}, 0)\big\}_{l=1}^{N_{\mathrm i}} \subset \Omega \times \{0\}$ for the initial condition, and $\mathcal{D}_{\mathrm b} = \big\{(\mathbf{x}^{(k)}_{\mathrm b}, t^{(k)}_{\mathrm b})\big\}_{k=1}^{N_{\mathrm b}} \subset \partial\Omega \times (0,T]$ for the boundary condition.
Using these sets, we define the following empirical loss function:
\begin{equation}
\mathrm{Loss}(\theta) 
= \lambda_{\mathrm{data}} \, \mathcal{L}_{\mathrm{data}}(\theta) 
+ \lambda_{\mathrm{pde}} \, \mathcal{L}_{\mathrm{pde}}(\theta) 
+ \lambda_{\mathrm ic} \, \mathcal{L}_{\mathrm{ic}}(\theta) 
+ \lambda_{\mathrm b} \, \mathcal{L}_{\mathrm{bc}}(\theta),
\label{eq:pinn_loss_full}
\end{equation}
with non-negative weights $\lambda_{\mathrm{data}}, \lambda_{\mathrm{pde}}, \lambda_{\mathrm{ic}}, \lambda_{\mathrm{bc}}$ balancing the different components. The individual loss terms are defined as:
\begin{align}
\mathcal{L}_{\mathrm{data}}(\theta) 
&= \frac{1}{N_{\mathrm{d}}} \sum_{i=1}^{N_{\mathrm{d}}} \left\| u_\theta(\mathbf{x}^{(i)}, t^{(i)}) - u^{(i)} \right\|_2^2, \label{eq:loss_data} \\
\mathcal{L}_{\mathrm{pde}}(\theta) 
&= \frac{1}{N_{\mathrm{r}}} \sum_{j=1}^{N_{\mathrm{r}}} \left\| \mathcal{L}[u_\theta](\mathbf{x}^{(j)}_{\mathrm{r}}, t^{(j)}_{\mathrm{r}}) - f(\mathbf{x}^{(j)}_{\mathrm{r}}, t^{(j)}_{\mathrm{r}}) \right\|_2^2, \\
\mathcal{L}_{\mathrm{ic}}(\theta) 
&= \frac{1}{N_{\mathrm{i}}} \sum_{l=1}^{N_{\mathrm{i}}} \left\| u_\theta(\mathbf{x}^{(l)}_{\mathrm{i}}, 0) - u_0(\mathbf{x}^{(l)}_{\mathrm{i}}) \right\|_2^2, \\
\mathcal{L}_{\mathrm{bc}}(\theta) 
&= \frac{1}{N_{\mathrm{b}}} \sum_{k=1}^{N_{\mathrm{b}}} \left\| \mathcal{B}[u_\theta](\mathbf{x}^{(k)}_{\mathrm{b}}, t^{(k)}_{\mathrm{b}}) - g(\mathbf{x}^{(k)}_{\mathrm{b}}, t^{(k)}_{\mathrm{b}}) \right\|_2^2.
\end{align}
where $\|\cdot\|_2$ denotes the Euclidean norm; other norm choices, such as Sobolev norms~\cite{fischer2020sobolev}, can also be considered depending on the PDE context~\cite{jiao2024stabilized, son2023sobolev}.

The optimization objective is to determine the optimal parameters $\theta^*$ that minimize the total loss defined in~\eqref{eq:pinn_loss_full}, i.e., $\theta^* = \arg\min_{\theta \in \mathbb{R}^{d_\theta}} \mathrm{Loss}(\theta)$. The resulting network $u_{\theta^*}$ serves as a surrogate for the true solution $u$, trained to satisfy both the empirical data and the governing physical laws. An illustration of this learning framework is provided in Figure~\ref{fig:pinn}.

\begin{figure}[htbp]  % h=here, t=top, b=bottom, p=page of floats
  \centering
  \includegraphics[width=0.9\linewidth]{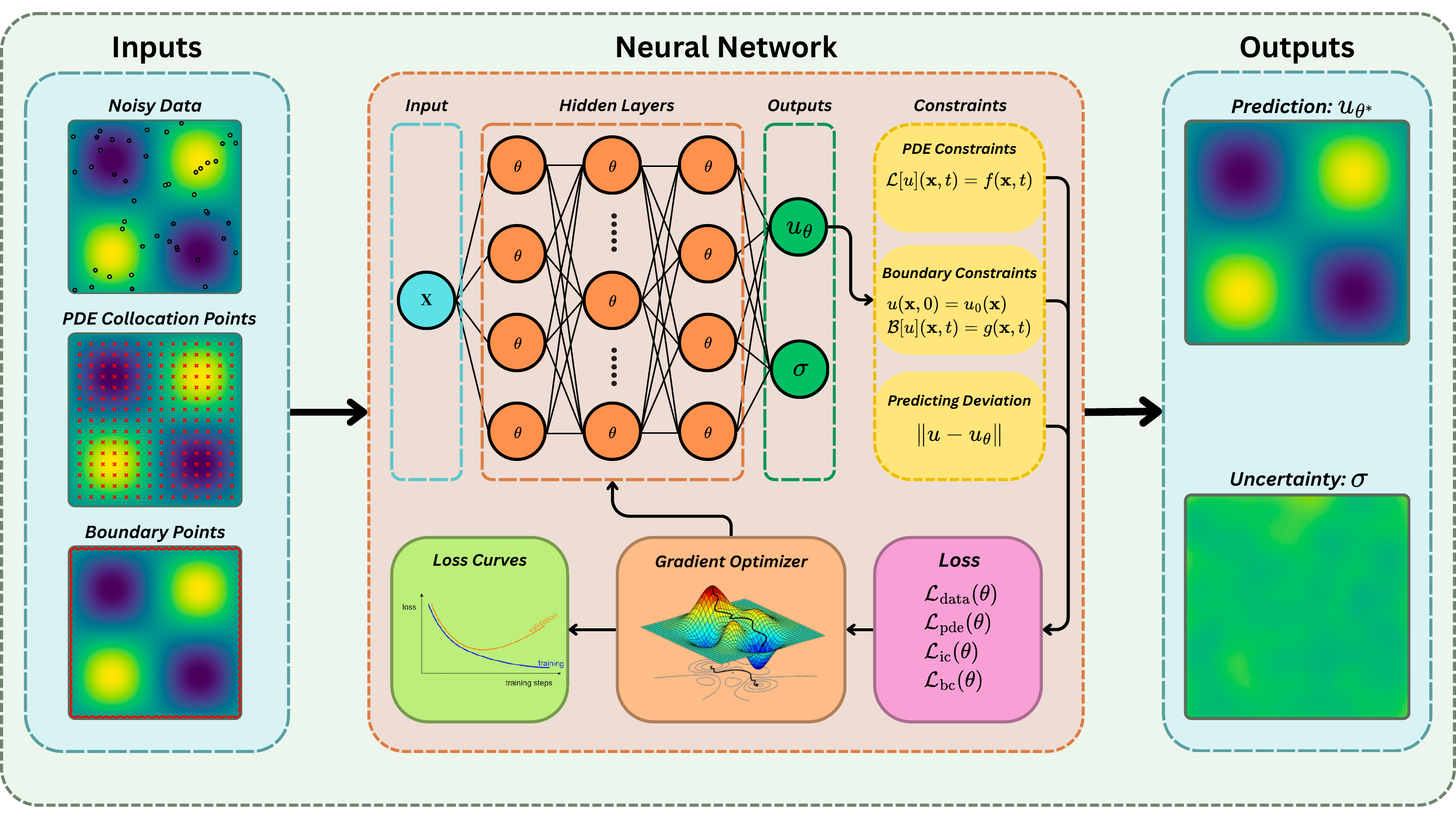}
  \caption{Workflow of uncertainty quantification in PINNs, illustrating the integration of noisy data, PDE and boundary constraints into a neural network to produce predictions and statistically valid uncertainty estimates.}
  \label{fig:pinn}
\end{figure}

\begin{remark}
Although the physics loss $\mathcal{L}_{\mathrm{pde}}$ enforces equation consistency, it is evaluated only at finitely many collocation points and thus may not uniquely determine the solution—particularly under sparse or high-dimensional sampling. As a result, minimizing $\mathcal{L}_{\mathrm{pde}}$ alone can yield non-physical or spurious solutions~\cite{rohrhofer2023data, hua2023physics}. The data loss $\mathcal{L}_{\mathrm{data}}$ is therefore crucial for anchoring the surrogate to observed values and improving generalization~\cite{mishra2023estimates}. In this work, we adopt a hybrid training strategy that combines both physical and data losses to ensure solution fidelity and robustness~\cite{raissi_physics-informed_2019, chen2020physics, raissi2020hidden}.
\end{remark}

\subsection{Analysis and Error Control}
\label{sec:sub:analysis}
% In this formulation, the solution $u$ is approximated by a parametric surrogate $u_\theta \in \mathcal{H}_\Theta$, where $\mathcal{H}_\Theta$ denotes a hypothesis space of continuous functions realized by a neural network with parameters $\theta \in \mathbb{R}^{d_\theta}$. 
% The network is trained to approximately satisfy the governing PDE as well as the associated initial and boundary conditions by minimizing a composite loss functional, thereby enforcing both data fidelity and physical consistency in a weak sense.
Motivated by the universal approximation theorem for neural networks~\cite{augustine2024survey}, the PINNs framework can be interpreted as a mesh-free, residual-minimization-based solver~\cite{raissi_physics-informed_2019, karniadakis2021physics}. Let $\mathcal{R}(u_\theta)(\mathbf{x},t) := \mathcal{L}[u_\theta](\mathbf{x},t) - f(\mathbf{x},t)$ denote the pointwise residual of the PDE. For sufficiently regular solutions, the residual $\mathcal{R}(u_\theta)$ serves as a practical indicator~\cite{mao2023physics}. A common {\it a posteriori} surrogate for the global error is given by the squared residual norm
\[
\| \mathcal{R}(u_\theta) \|_{L^2(\Omega \times (0,T))}^2 
\approx \frac{1}{N_{\rm r}} \sum^{N_{\rm r}}_{j=1} \left\| \mathcal{R}(u_\theta)(\mathbf{x}^{(j)}_{\rm r},t^{(j)}_{\rm r}) \right\|^2,
\]
which provides an empirical measure of how well the surrogate satisfies the governing equations.

However, such residual-based quantities are limited: they reflect local violations of the PDE but neither quantify predictive confidence in unseen regions nor yield statistically meaningful bounds on the true solution error $\|u - u_\theta\|$. Although recent studies have established a priori error estimates for PINNs under certain restrictive assumptions~\cite{mishra2023estimates}, these results remain largely theoretical and may not translate directly into practical uncertainty quantification.

% Recent numerical analyses of PINNs~\cite{mishra2023estimates, mishra2022estimates} have established \textit{a priori} error bounds of the form
% \[
% \|u_\theta - u\|_{\rm H^1(\Omega)} \leq C d_\theta^{-1/4} \|u\|_{\rm H^2(\Omega)},
% \]
% under suitable regularity assumptions and the spatial dimension $d$ is large, where $d_\theta$ denotes the number of trainable parameters, and $C > 0$ is a constant independent of network parameters. 

% In the context of PINNs, however, the hypothesis space $\mathcal{H}_\Theta$ is nonlinear, high-dimensional, and parameterized by a neural network architecture. These features render both \textit{a priori} and \textit{a posteriori} analysis difficult to formalize rigorously. 

This motivates the development of UQ frameworks with finite-sample, distribution-free guarantees. We focus on characterizing the uncertainty of $u_\theta$ in regions weakly constrained by data or physics, and propose a conformal prediction-based approach that provides statistically valid coverage without distributional assumptions.

% -----------------------------------------------------------
% Heuristics & Conformal Prediction
% -----------------------------------------------------------
\section{Conformal Prediction for Uncertainty Quantification}
\label{sec:cp&uq}

\subsection{Heuristic Uncertainty Quantification}
\label{sec:sub:original_uq}

To establish uncalibrated baselines for comparison with our CP approach, we consider five heuristic uncertainty estimation strategies within the PINN framework. These methods approximate the uncertainty of the surrogate solution $u_\theta(x):=u_\theta(\mathbf{x},t)$ (denoted simply as $x$) on the dataset $\mathcal D_{\mathrm{data}}=\{x_i=(\mathbf{x}_i,t_i)\}_{i=1}^{N_{\rm d}}$, without formal statistical calibration. The baselines are grouped into three categories:

\begin{enumerate}[label=(\roman*)]
\item Geometric or latent-space distance: A non-Bayesian heuristic that estimates uncertainty based on the distance between a test point and the training data manifold in latent space~\cite{hu2022robust}.

\item Monte Carlo dropout~\cite{gal_dropout_2016}: A Bayesian approximation technique uses random dropout during inference, producing an ensemble of outputs that reflect predictive variance.

\item Bayesian posterior sampling: This category includes methods such as variational inference (VI)~\cite{blei2017variational} and Hamiltonian Monte Carlo (HMC)~\cite{yang_b-pinns_2021}, which aim to approximate the posterior distribution to generate predictive uncertainty estimates. VI provides a tractable but approximate posterior via optimization, while HMC offers more accurate sampling at greater computational cost.
\end{enumerate}

All five methods aim to estimate predictive uncertainty through a surrogate function $\sigma:\mathbb{R}^{d} \longrightarrow \mathbb{R}{\ge 0}$, which assigns a raw uncertainty score to each new input $x_{\rm new}$ based on model-specific heuristics. While these scores yield empirical uncertainty bands, they generally lack formal coverage guarantees~\cite{mousavi2017heuristics}. Within our framework, such raw scores serve as inputs to the conformal prediction procedure, which converts them into calibrated prediction intervals. The construction of each baseline is detailed below.

\subsubsection{Distance-Based Uncertainty Estimation}
\label{sec:sub:distance}

We emphasize that the deterministic forward pass of the neural network yields the predictive mean; hence the subsequent discussion focuses on the strategies for constructing meaningful uncertainty estimates.

\paragraph{Geometric Distance (GD)}
The geometric distance is a widely used method for estimating uncertainty based on geometric complexity in the input space. Given a training dataset $\mathcal D_{\mathrm{data}} $, we define the uncertainty at a test point $x_{\rm new}$ using its distance to nearby training samples.

Specifically, let $\mathcal{N}_{K}(x_{\rm new}) \subset \mathcal D_{\mathrm{data}}$ denote the set of the $K$ nearest neighbors of $x_{\rm new}$ under the Euclidean norm. Alternative distance metrics may be employed depending on the geometry of the input space. The GD-based uncertainty estimate is then defined as the average distance between $x_{\rm new}$ and its $K$ nearest neighbors:
\begin{equation}
\sigma_{\textrm {GD}}^2(x_{\rm new})
:= \frac{1}{K} \sum_{x_k \in \mathcal N_{K}(x_{\rm new})} \lVert x_{\rm new} - x_k \rVert_2^2. 
% \quad \textrm{with} \quad x_j = \mathop{\arg\min}_{x' \in \mathcal D_{\mathrm{data}} \setminus {x_1, \dots, x_{j-1}}} \lVert x_{\rm new} - x' \rVert_2
\label{eq:fd_hu}
\end{equation}
By construction, $\sigma_{\textrm{GD}}(x_{\rm new})$ decreases in densely sampled areas and increases in regions far from training data, thus providing a simple proxy for uncertainty caused by data sparsity.

\paragraph{Latent-space Distance (LD)} 
To improve the geometric fidelity of distance-based uncertainty estimates, we evaluate distances in the latent space of the network, represented by the penultimate layer mapping $h_\theta$.
% This embedding offers a geometric space where Euclidean distance more accurately reflects predictive similarity.

Recall the definition of $\mathcal{N}_{K}(x_{\rm new}) \subset \mathcal D_{\mathrm{data}}$ denote the set of the $K$ nearest neighbors of $x_{\rm new}$ under the Euclidean norm. The LD-based uncertainty is then defined as:
\begin{equation}
\sigma_{\textrm{LD}}^2(x_{\rm new}) := \frac{1}{K} \sum_{x_k \in \mathcal N_{\mathrm{k}}(x_{\rm new})} \big\lVert h_{\theta}(x_{\rm new}) - h_{\theta}(x_k) \big\rVert_2^2. 
% \quad \textrm{with } x_j = \mathop{\arg\min}_{x' \in \mathcal D_{\mathrm{data}} \setminus \{x_1, \dots, x_{j-1}\}} \lVert h_{\theta}(x_{new}) - h_{\theta}(x') \rVert_2.
\label{eq:ld_hu}
\end{equation}
% where the set of nearest neighbors $\mathcal N_k^{\text{latent}}(x) = {x_j^\ast}{j=1}^{k}$ is constructed similarly to the input-space case, but with distances evaluated in the latent space:
% \begin{equation}
% \end{equation}
% Let $\mathcal{H}_{k} := \{h_{\theta}(x_j)\}_{j=1}^{k} \subset \mathbb{R}^{m}$ denote the latent representations of the training inputs. 
The uncertainty score $\sigma_{\textrm{LD}}(x_{\rm new})$ reflects the local density and benefits from the expressiveness of the learned Geometric space. However, in high-dimensional latent spaces, Euclidean distances become less informative due to the “curse of dimensionality”, and alternative metrics or dimensionality reduction techniques (e.g., PCA~\cite{abdi2010principal}, t-SNE~\cite{maaten2008visualizing}) may be required to restore discriminative power.

% \yf{
% For both the feature-space and latent-space distance-based uncertainty heuristics UQ model, the predictive mean is given by the deterministic forward pass of the neural network:
% \begin{align}
% \mu_{\text{distance}} = \mathrm{NN.forward}(x).
% \end{align}
% }

\subsubsection{Monte Carlo (MC) Dropout (DO)}
\label{sec:sub:mc_do}

MC-DO offers a scalable approximation to Bayesian inference in neural networks~\cite{gal_dropout_2016}. By interpreting dropout as stochasticity in the weights, each forward pass corresponds to a sample from a variational posterior. Retaining dropout at inference enables approximate posterior sampling and uncertainty estimation via Monte Carlo statistics.

Let $\{m_{\mathrm{DO}}^{(n)}\}_{n=1}^{N_{\mathrm{MC}}}$ denote a collection of $N_{\mathrm{MC}}$ independent dropout masks sampled during inference. For a test sample $x_{\mathrm{new}}$, each stochastic forward pass yields a realization
\[
u_\theta^{(n)}(x_{\mathrm{new}}) := f_{\theta, m_{\mathrm{DO}}^{(n)}}(x_{\mathrm{new}}),
\]
where $f_{\theta, m}$ denotes the network output with parameters $\theta$ and dropout mask $m$. The empirical mean and predictive variance are then estimated by:
\begin{equation}
\mu_{\mathrm{DO}}(x_{\rm new}) := \frac{1}{N_{\mathrm{MC}}} \sum_{n=1}^{N_{\mathrm{MC}}} u_\theta^{(n)}(x_{\rm new}), 
\quad
{\sigma}_{\mathrm{DO}}^2(x_{\rm new}) := \frac{1}{N_{\mathrm{MC}}} \sum_{n=1}^{N_{\mathrm{MC}}} \left\| u_\theta^{(n)}(x_{\rm new}) - {\mu}_{\mathrm{DO}}(x_{\rm new}) \right\|_2^2.
\label{eq:do_mean_var}
\end{equation}
Here, ${\sigma}_{\mathrm{DO}}(x_{\rm test})$ serves as a pointwise estimate of uncertainty associated with the prediction at $x_{\rm test}$. This method offers a computationally efficient alternative to full Bayesian inference, while still capturing model uncertainty induced by limited training data or structural mismatch.

\subsubsection{Bayesian PINNs (B-PINNs)}
\label{sec:sub:bayesian_posterior_sampling}

In the Bayesian framework~\cite{bernardo1994bayesian}, epistemic uncertainty is captured by placing a prior distribution $p_0(\theta)$ over $\theta$ and updating it in light of observed data using Bayes’ theorem~\cite{berrar2018bayes}. When applied to PINNs, this results in the so-called B-PINNs~\cite{yang_b-pinns_2021}, whose posterior distribution is given by
\begin{equation}
p(\theta | \mathcal{D}_{\mathrm{data}}) = \frac{p(\mathcal{D}_{\mathrm{data}} | \theta)\, p_0(\theta)}{\int p(\mathcal{D}_{\mathrm{data}} | \theta)\, p_0(\theta)\, \mathrm{d}\theta},
\label{eq:posterior_exact}
\end{equation}
where $p(\mathcal{D}_{\mathrm{data}}|\theta)$ is the likelihood function and the denominator represents the model evidence.

The posterior in~\eqref{eq:posterior_exact} is generally intractable due to the high-dimensional integral in the denominator. Consequently, approximate inference methods are required. In this work, we consider two widely used techniques: Variational Inference and Hamiltonian Monte Carlo, described below.

\paragraph{Variational Inference (VI)}
\label{par:bayesian_posterior_sampling}

VI approximates the intractable posterior with a tractable family of distributions by solving an optimization problem. In this work, we assume a fully factorized Gaussian approximation,
\begin{equation}
  q_{\phi}(\theta) =
  \prod_{j=1}^{d_\theta}
  \mathcal N\!\bigl(\theta_j | \mu_j, \sigma_j^{2}\bigr),
  \qquad
  \sigma_j = \operatorname{softplus}(\rho_j),
  \label{eq:mf_gaussian}
\end{equation}
where $d_\theta$ is the parameter dimension and $\phi=\{(\mu_j,\rho_j)\}_{j=1}^{d_\theta}$ are the variational parameters. The softplus reparameterization ensures strictly positive standard deviations~\cite{blundell_weight_2015}. The variational parameters are obtained by minimizing the negative evidence lower bound (ELBO),
\begin{align}
  \min_{\phi}\ - \mathcal L_{\mathrm{ELBO}}(\phi) 
  &:= -
     \underbrace{\mathbb E_{q_{\phi}}
       \!\bigl[\log p(\mathcal D_{\mathrm{data}}|\theta)\bigr]}_{\text{expected log-likelihood}}
     \;+\;
     \underbrace{\mathrm{KL}\bigl(q_{\phi}(\theta)\,||\,p_{0}(\theta)\bigr)}_{\text{complexity penalty}},
  \label{eq:elbo}
\end{align}
where the first term encourages data fidelity while the KL term regularizes the solution towards the prior. Once trained, the surrogate posterior $q_{\phi}(\theta)$ enables efficient sampling: we draw $M$ parameter samples $\{\theta_i\}_{i=1}^M$ to approximate the predictive distribution. Details of training and inference are provided in~\ref{sec:apd:vi}.

\paragraph{Hamiltonian Monte Carlo (HMC)}
\label{sec:sub:hmc}
Different from VI, HMC approximate the true posterior distribution by directly sampling $\theta$ from a surrogate system (Hamiltonian system) with the Metropolis-Hasting Algorithm~\cite{neal_mcmc_2012, betancourt2017conceptual}. It first construct a Hamiltonian system:
\begin{equation}
\label{eq:hamiltonian_full}
H(\theta,r)=U(\theta)+V(r),
\quad 
\text{where}
\quad
\begin{aligned}
    U(\theta)
  &:= -\log p(\mathcal D_{\mathrm{data}}|\theta) - \log p_0(\theta) + \text{const} \\
  V(r)
  &:= \tfrac{1}{2}\, r^{\mathsf T} \mathbb{M}^{-1} r
\end{aligned},
\end{equation}
where the potential energy $U(\theta)$ encode the parameters' posterior formulation~\eqref{eq:posterior_exact} in the potential energy function and $V(r)$ is the fictional kinetic component, in which \( r \in \mathbb{R}^{d_\theta} \) is the momentum variable, and \( \mathbb{M} \in \mathbb{R}^{d_\theta \times d_\theta} \) is the mass matrix. The parameters are sampled from a Hamiltonian dynamics, which will be given in details in~\ref{sec:apd:hmc}. With the drawn samples we form the parameter samples set denoted by $\Theta_{\mathrm{HMC}}=\{\theta_i\}_{i=0}^M$.

During prediction, given a query point $x_{\mathrm{new}}$, the B-PINNs compute the predictive mean and variance as
\begin{equation}
    \begin{aligned}
    \mu_{\mathrm{BAY}}(x_{\mathrm{new}})
      &= \frac{1}{M}\sum_{m=1}^{M} f_{\theta^{(m)}}(x_{\mathrm{new}}), \\
    \sigma^2_{\mathrm{BAY}}(x_{\mathrm{new}})
      &= \frac{1}{M}\sum_{m=1}^{M}
         \bigl\| f_{\theta^{(m)}}(x_{\mathrm{new}})
         - \mu_{\mathrm{BAY}}(x_{\mathrm{new}}) \bigr\|^2_2,
    \end{aligned}
    \label{eq:vimc_mv}
\end{equation}
where the parameter samples $\theta^{(m)}$ are drawn according to the chosen inference method:
\[
\theta^{(m)} \sim 
\begin{cases}
  q_{\phi}(\theta), & \text{Variational Inference (VI)} \\
  \Theta_{\mathrm{HMC}}, & \text{Hamiltonian Monte Carlo (HMC)}.
\end{cases}
\]

%────────────────────────────────────────────────────────────
\subsection{Conformal Prediction}
\label{sec:sub:cp}
%────────────────────────────────────────────────────────────

Conformal prediction (CP) provides distribution-free prediction intervals with guaranteed finite-sample coverage under minimal assumptions~\cite{shafer2008tutorial, angelopoulos2023conformal}. 
It requires an additional labeled calibration dataset 
$\mathcal{D}_{\mathrm{cal}}=\{(x_i, u_i)\}_{i=1}^{N_{\rm c}}$ 
that is independent of the training set.

\paragraph{Vanilla CP}
In its basic form, CP calibrates model predictions by constructing nonconformity scores on $\mathcal{D}_{\mathrm{cal}}$, thereby yielding prediction intervals with statistically valid coverage guarantees. 

With a trained deterministic predictor 
$u_{\theta}:\mathcal X \to \mathbb R$ (e.g., the PINN considered here), 
we define nonconformity scores on the calibration set $\mathcal D_{\mathrm{cal}}$. 
For each calibration pair, the score is taken as the absolute residual,
\begin{equation}
  r_i = \bigl|u_i - u_{\theta}(x_i)\bigr|,
  \qquad
  R = \{r_i\}_{i=1}^{N_{\mathrm c}}.
  \label{eq:score_vanilla}
\end{equation}

Let $q_{1-\alpha}$ denote the 
$\lceil (1-\alpha)(N_{\mathrm c}+1)\rceil$-th smallest element of $R$. 
Then, for a new input $x_{\mathrm{new}}$, the vanilla $(1-\alpha)$ conformal 
prediction interval is
\begin{equation}
  I_{1-\alpha}(x_{\mathrm{new}})
  =
  \Bigl[
      u_{\theta}(x_{\mathrm{new}}) - q_{1-\alpha},\;
      u_{\theta}(x_{\mathrm{new}}) + q_{1-\alpha}
  \Bigr],
  \label{eq:vanilla_cp_interval}
\end{equation}
which by construction guarantees the finite-sample coverage (cf.~Theorem~\ref{thm:cp}, see also~\cite{angelopoulos_gentle_2022})
\[
  \mathbb{P}\bigl\{u_{\mathrm{new}}\in I_{1-\alpha}(x_{\mathrm{new}})\bigr\}
  \ge 1-\alpha.
\]

\paragraph{CP}

Vanilla CP computes absolute-residual scores $r_i=\lvert u_i - u_\theta(x_i)\rvert$, thereby ignoring heteroskedasticity and reducing robustness~\cite{hu2022robust}. CP mitigates this issue by \emph{normalizing} residuals with a positive scale estimate, 
so that interval widths adapt to varying noise levels while preserving finite-sample validity. 
Achieving full local adaptivity would further require $x$-dependent quantiles, 
as in our local CP extension (see Section~\ref{sec:extension}).

Concretely, CP employs a scale function 
$\sigma:\mathcal{X}\to\mathbb{R}_{>0}$, typically provided by the baseline uncertainty model, 
to normalize residuals. For each calibration pair $(x_i,u_i)\in\mathcal{D}_{\mathrm{cal}}$, the scaled nonconformity score is
\begin{equation}
  s_i = \frac{\lvert u_i - u_\theta(x_i)\rvert}{\sigma(x_i)},
  \qquad
  S = \{s_i\}_{i=1}^{N_{\mathrm c}}.
  \label{eq:score_scaled}
\end{equation}

Let $q^{\mathrm{cp}}_{1-\alpha}$ denote the 
$\lceil (1-\alpha)(N_{\mathrm c}+1)\rceil$-th smallest element of $S$. 
The $(1-\alpha)$ CP interval at a new input $x_{\mathrm{new}}$ is then
\begin{equation}
  I^{\mathrm{cp}}_{1-\alpha}(x_{\mathrm{new}})
  =
  \Bigl[
    u_\theta(x_{\mathrm{new}}) - q^{\mathrm{cp}}_{1-\alpha}\sigma(x_{\mathrm{new}}),\;
    u_\theta(x_{\mathrm{new}}) + q^{\mathrm{cp}}_{1-\alpha}\sigma(x_{\mathrm{new}})
  \Bigr],
  \label{eq:scaled_cp_interval}
\end{equation}
which retains the finite-sample, distribution-free coverage guarantee of vanilla CP, 
while adapting interval widths to heteroskedasticity compared to~\eqref{eq:vanilla_cp_interval}. The following Theorem establishes the theoretical foundation of CP. Unless otherwise noted, CP serves as our default calibration method for heuristic uncertainties throughout this work.

\begin{theorem}[{\cite[Theorem~2]{romano2019conformal}}]
\label{thm:cp}
Let $\{(x_i,u_i)\}_{i=1}^{n} \cup (x_{\rm new}, u_{\rm new})$ be an exchangeable sequence drawn i.i.d.\ from the data distribution. 
% Let $u_\theta$ be a fixed predictor, and $\hat\sigma$ a fixed positive function. 
Then the interval $I^{\mathrm{cp}}_{1-\alpha}(x_{\rm new})$ defined by~\eqref{eq:scaled_cp_interval} satisfies
\[
  \mathbb{P}\bigl(u_{\rm new}\in I^{\mathrm{cp}}_{1-\alpha}(x_{\rm new})\bigr)\ge1-\alpha.
\]
Moreover, if the scaled scores $|u-u_\theta(x)|/\sigma(x)$ are continuous, then
\[
  \mathbb{P}\bigl(u_{\rm new}\in I^{\mathrm{cp}}_{1-\alpha}(x_{\rm new})\bigr)\le1-\alpha+\tfrac{1}{n+1}.
\]
\end{theorem}

% To apply this theorem, we substitute $(x_{\rm test}, u_{\rm test})$ for $(x_{n+1}, u_{n+1})$.  

%────────────────────────────────────────────────────────────
\subsection{Evaluation Metrics}
\label{sec:metrics}
%────────────────────────────────────────────────────────────

We evaluate predictive uncertainty using two complementary metrics: \emph{empirical coverage} and \emph{average coverage deviation (ACD)}. These metrics assess statistical validity both at a target significance level and across a range of levels, providing a comprehensive evaluation of calibration. The effectiveness of CP will be systematically examined through these criteria in Section~\ref{sec:numerics}.

\paragraph{Empirical Coverage}

The empirical coverage $\hat{c}$ is defined as the proportion of ground-truth targets that fall within their corresponding prediction intervals. 
Let $\mathcal{D}_{\mathrm{test}}=\{(x_i,u_i)\}_{i=1}^{N_{\rm test}}$ denote the held-out test set, and let $[L_i,U_i]$ be the predicted interval at expected coverage level $(1-\alpha)$ for input $x_i$. 
Then $\hat{c}$ is given by
\begin{equation}
\hat{c} := \frac{1}{N_{\rm test}} \sum_{i=1}^{N_{\rm test}} \mathbf{1} \left\{ L_i \le u_i \le U_i \right\},
\label{eq:coverage}
\end{equation}
where $\mathbf{1}{\cdot}$ is the indicator function. A well-calibrated model should yield $\hat{c}$ close to the expected target level $1-\alpha$. For instance, when $\alpha=0.05$, the desired coverage is $0.95$, and a model is regarded as well calibrated at this level if its empirical coverage $\hat{c}$ is close to $0.95$. Deviations from the target level indicate miscalibration: if $\hat{c}<1-\alpha$, the intervals are too narrow and lead to \emph{under-coverage}; if $\hat{c}>1-\alpha$, the intervals are too wide and lead to \emph{over-coverage}. Both scenarios suggest that the model need further calibration.

% \paragraph{Sharpness}

% Sharpness measures the average width of prediction intervals, reflecting how informative they are:
% \begin{equation}
% \text{Sharpness} := \frac{1}{N_{\rm test}} \sum_{i=1}^{N_{\rm test}} \left( U_i - L_i \right).
% \label{eq:sharpness}
% \end{equation}
% Narrower intervals indicate more confident predictions, but sharpness must be interpreted alongside coverage to ensure reliability. High-quality uncertainty quantification is therefore characterized by both accurate coverage and reasonable sharpness.

\paragraph{Average Coverage Deviation (ACD)}

To evaluate calibration across multiple significance levels, we use the \emph{average coverage deviation}, which aggregates the discrepancy between empirical and expected coverage over a range of significance levels. Let $\{\alpha_k\}_{k=1}^K \subset (0,1)$ be a set of miscoverage levels. For each $\alpha_k$, the expected coverage is $1-\alpha_k$, and the empirical coverage $\hat{c}_k$ is computed via Eq.~\eqref{eq:coverage}. The ACD is then defined as
\begin{equation}
\text{ACD} := \frac{1}{K} \sum_{k=1}^{K} \big| \hat{c}_k - (1 - \alpha_k) \big|.
\label{eq:acd}
\end{equation}
Lower values indicate better calibration uniformly across the grid of significance levels. 
An ideal model would achieve $\mathrm{ACD}=0$, corresponding to exact agreement between empirical and expected coverage. 
In practice, $\mathrm{ACD}$ can be interpreted as the average absolute gap (in coverage probability) between what the model achieves and what it targets across different significance levels, thus providing a single scalar summary of calibration quality over the entire range of $\alpha$.

\section{Numerical Experiments}
\label{sec:numerics}

In this section, we evaluate the effectiveness of CP for calibrating prediction intervals in PINNs. Three benchmark PDEs are considered: (a) the 1D Poisson equation (Section~\ref{sec:1dp}), (b) the 2D Allen--Cahn equation (Section~\ref{sec:ac2d}), and (c) the 3D Helmholtz equation (Section~\ref{sec:3dh}). For each case, we compare five heuristic UQ methods, both before and after CP calibration: (i) geometric distance, (ii) latent distance, (iii) dropout, (iv) variational-inference (VI), and (v) Hamiltonian Monte Carlo (HMC).  

% \paragraph{Collocation and Boundary Points} 
% To enforce the governing PDE constraints during model training,} we uniformly distribute $250$ collocation points in the interior of the spatial domain (doubled for 3D example), together with $1000$ boundary points for all three PDE problems. Dirichlet boundary conditions are imposed as hard constraints to ensure exact satisfaction at the domain boundaries.

% \paragraph{Neural Network Architectures and Training} 
For the 1D Poisson problem, we adopt a four-layer multilayer perceptron with hidden widths $[25,\, 35,\, 35,\, 25]$. 
The network is expanded to $[16,\, 32,\, 64,\, 64,\, 64,\, 32,\, 16]$ for the 2D Allen--Cahn benchmarks, and for the more challenging 3D Helmholtz equation, we use a deeper and wider network with widths 
$[32,\,64,\,128,\,128,\,128,\,64,\,32]$.
% The same architecture is employed for the 2D benchmarks to ensure comparability across problems of different dimensionality. 
The hidden layers are activated by $\tanh$, and the network weights are initialized according to the Xavier scheme~\cite{glorot2010understanding}, which is used throughout our experiments to maintain consistency and stable optimization. Training is performed using the Adam optimizer, followed by a step-wise learning rate scheduling for convergence. 
The detailed hyperparameter settings, including learning rate schedules, training epochs, and loss weights, are summarized in Table~\ref{tab:training_parameter} in~\ref{sec:apd:num}. 

% \paragraph{Baseline UQ Methods} For VI, we follow the implementation of Yang~\cite{yang_b-pinns_2021}, employing a mean-field Gaussian variational family and optimizing it via reparameterization gradients. Dropout uses a dropout rate of XXX applied to all hidden layers during both training and prediction. Feature and latent distances are computed from the penultimate hidden layer and the encoder's latent space, respectively, using Euclidean distance to the XXX nearest training points.

% \paragraph{Implementation} 
All experiments are implemented using the \texttt{PyTorch} deep learning framework. Training and evaluation are carried out on a laptop equipped with an Apple M4 Pro processor and 24\,GB of memory. The full source code, including PDE solvers, data generation scripts, and CP calibration routines, is publicly available at our \url{https://github.com/RoyYu0509/LocalCP4PINN}.

% \paragraph{Data}

% \paragraph{Prediction}
% For each baseline model, we train on the noisy samples and enforce PDE residuals on interior (boundary and initial) collocation points. Unless noted otherwise, two-sided $1-\alpha=0.95$ prediction bands are formed as
% \[
% \hat u(x)\ \pm\ z_{1-\alpha/2}\,\hat\sigma_{raw}(x),\quad \alpha=0.05,
% \]
% assuming Gaussian noise for the raw bands. We then apply scaled conformal prediction (CP) to obtain calibrated bands.

% \paragraph{Evaluation Metrics}

% We assess prediction quality both visually and quantitatively. Visually, we plot the true field, predictive mean, interval width, and the expected vs.\ empirical coverage curve across $\alpha\in[0,1]$; the ideal coverage is the line $y=1-\alpha$. Quantitatively, we report: (i) empirical coverage at expected $1-\alpha=0.95$, (ii) Average coverage deviation (MAE between empirical and ideal coverage across the $\alpha$-grid), and (iii) sharpness (mean interval width).

\subsection{1D Poisson Equation}
\label{sec:1dp}

We begin with a simple 1D Poisson problem to illustrate the impact of CP on uncertainty calibration:
\[
u''(x) = f(x), \quad x \in [0,1], \qquad u(0) = u(1) = 0,
\]
with $f(x) = -\pi^{2} \sin(\pi x)$ and exact solution $u^*(x) = \sin(\pi x)$. 
B-PINNs (cf.~Section~\ref{sec:sub:bayesian_posterior_sampling}) are implemented using VI, though other heuristic approaches would be expected to behave similarly. 

We generate $60$ training samples by drawing inputs uniformly at random from the domain and evaluating the corresponding targets from the exact solution. An additional $30$ samples are reserved for calibration using the same procedure. Independent Gaussian noise with zero mean and standard deviation $\sigma=0.15$ is added to both sets to mimic measurement error. The method remains valid under any exchangeable sampling scheme. In addition, $200$ uniformly spaced collocation points are placed in the interior of the domain, with Dirichlet boundary conditions enforced at the endpoints.

\begin{figure}[h]
  \centering
  \includegraphics[width=0.95\linewidth]{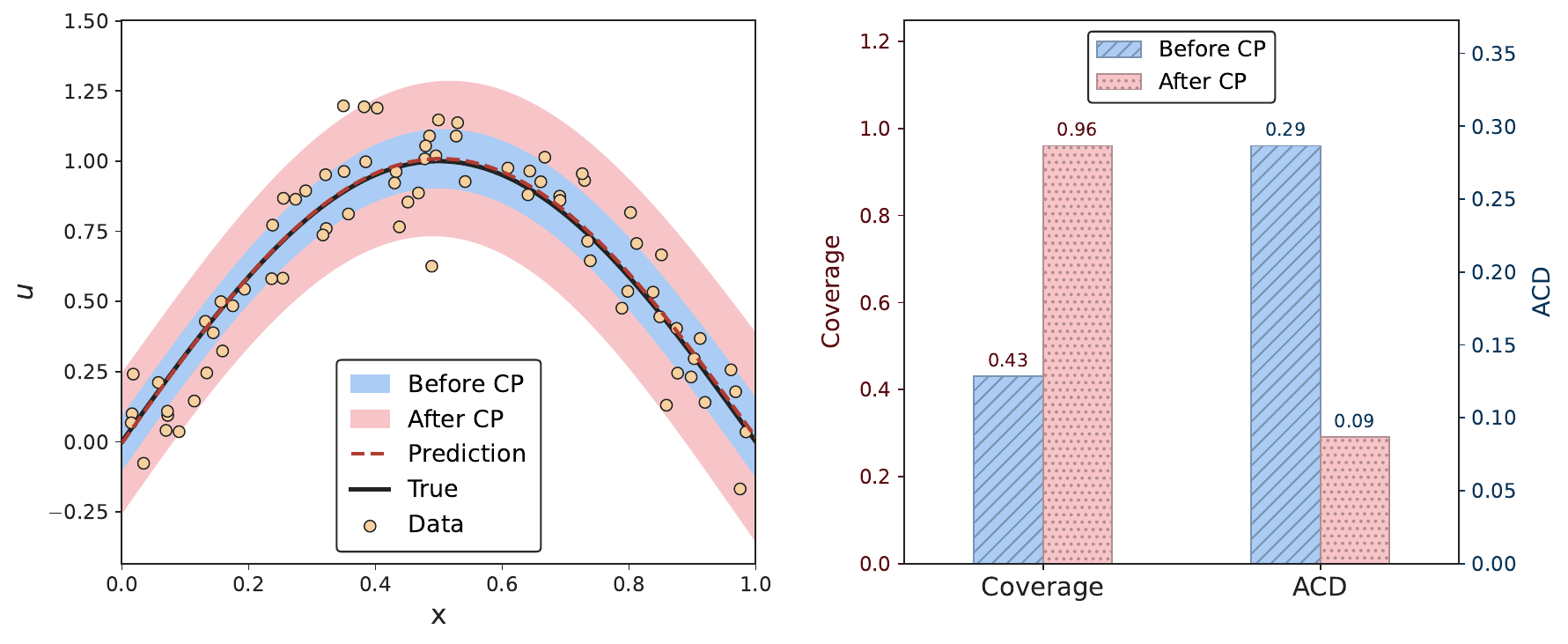}
\caption{
Variational Inference B-PINN uncertainty calibration on the 1D Poisson problem at $\alpha=0.05$. 
\textbf{Left:} Prediction intervals before (red) and after (blue) CP. 
\textbf{Right:} Comparison of empirical coverage and ACD. ACD is computed over 19 equally spaced $\alpha_k \in [0.05, 0.95]$ on the test dataset.}
  \label{fig:1Ddemo}
\end{figure}

Figure~\ref{fig:1Ddemo} reports results at the expected confidence level $1-\alpha = 0.95$. The left panel shows prediction intervals before and after calibration. The naïve B-PINN intervals substantially under-cover the truth, reflecting over-confident uncertainty estimates. After applying CP, the intervals achieve near-expected coverage, while this improvement is accompanied by expanded intervals (less informative), the trade-off is consistent with the need to correct the systematic miscalibration. Crucially, this improvement is obtained {\it post hoc}, without retraining the surrogate model. The right panel summarizes additional evaluation metrics, all of which confirm the benefit of CP: empirical coverage closely matches the expected coverage level and average coverage deviation is reduced. This toy example therefore demonstrates that CP can reliably transform unreliable uncertainty estimates into calibrated prediction intervals.

\subsection{2D Allen--Cahn Equation}
\label{sec:ac2d}

We next examine the steady Allen--Cahn equation, following~\cite{yang_b-pinns_2021}, posed on the square domain $\Omega = (-1,1)\times(-1,1)$ with Dirichlet boundary conditions prescribed by the exact solution:
\begin{align}
  \lambda\,\Delta u(x,y) + u(x,y)\big(u(x,y)^2 - 1\big) &= f(x,y), && (x,y)\in\Omega, \\
  u(x,y) &= u^*(x,y), && (x,y)\in\partial\Omega ,
\end{align}
where $\lambda=0.05$, $u^*(x,y)=\sin(\pi x)\sin(\pi y)$, and the forcing term $f$ is obtained analytically by substituting $u^*$ into the PDE.
We generate 500 i.i.d. synthetic observations (with either Latin hypercube sampling~\cite{helton2003latin} or i.i.d. uniform; we use i.i.d. uniform in the reported runs), partitioned into 300 training, 100 calibration, and 100 testing samples. Measurement noise is introduced in the same manner as described in Section~\ref{sec:1dp}, but with $\sigma = 0.05$. $1{,}024$ collocation points and $800$ boundary points are placed evenly in the interior of the domain and on the boundaries to enforce physics. The dense boundary points allocation is used to ensure accurate satisfaction of boundary constraints, which are critical for the overall solution quality.
% The choice of sampling distribution can be arbitrary, provided it preserves exchangeability

\begin{figure}[h]
  \centering
  \includegraphics[width=0.95\linewidth]{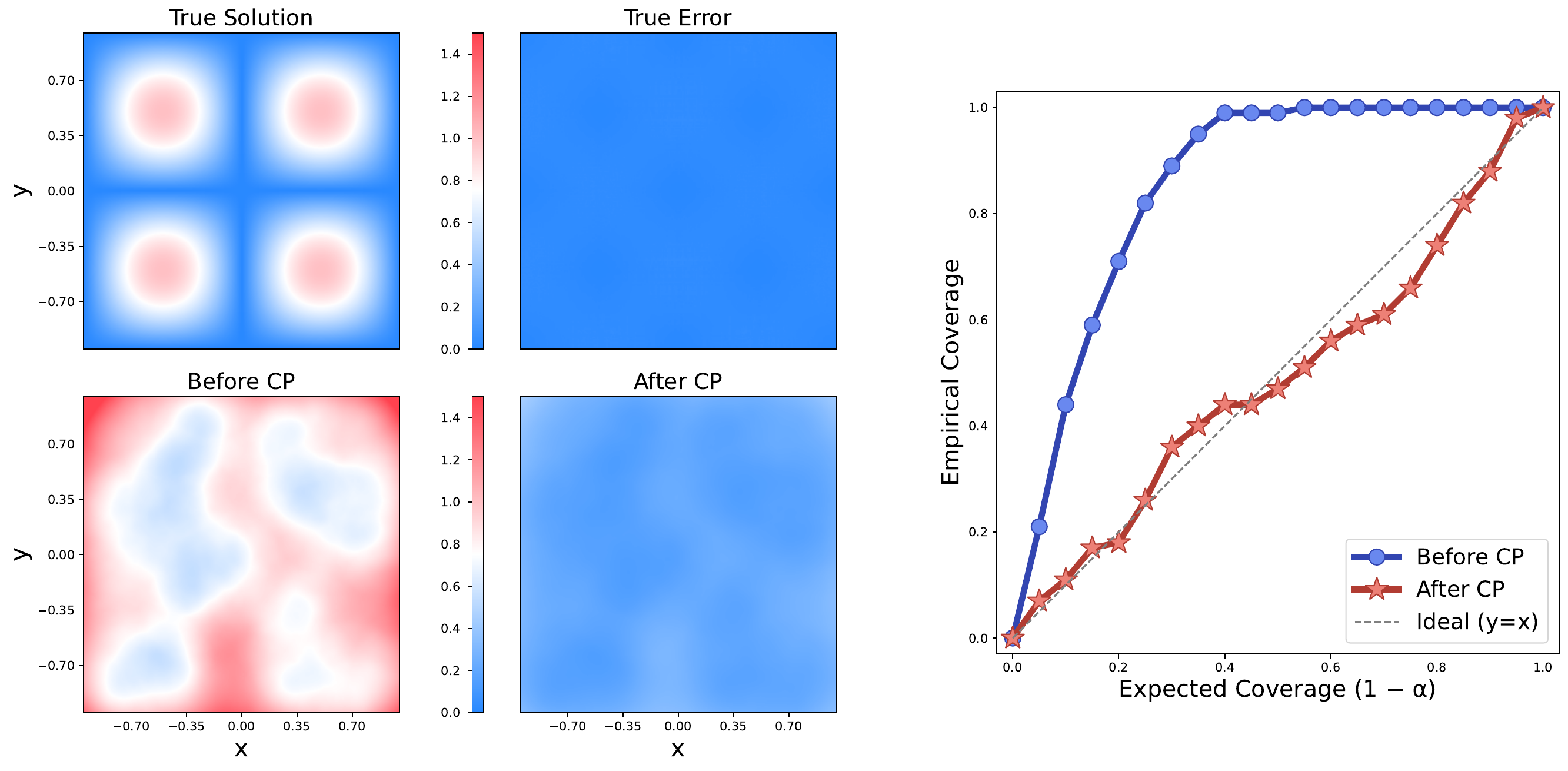}
  \caption{Geometric-distance PINN uncertainty calibration for the 2D Allen–Cahn equation. \textbf{Left:} True solution in absolute value, surrogate prediction, and error distributions before and after CP at $\alpha=0.05$. \textbf{Right:} Empirical coverage versus expected coverage across varying $\alpha$. }
  \label{fig:ac_fd_model}
\end{figure}

\begin{figure}[h]
  \centering
  \includegraphics[width=1.0\linewidth]{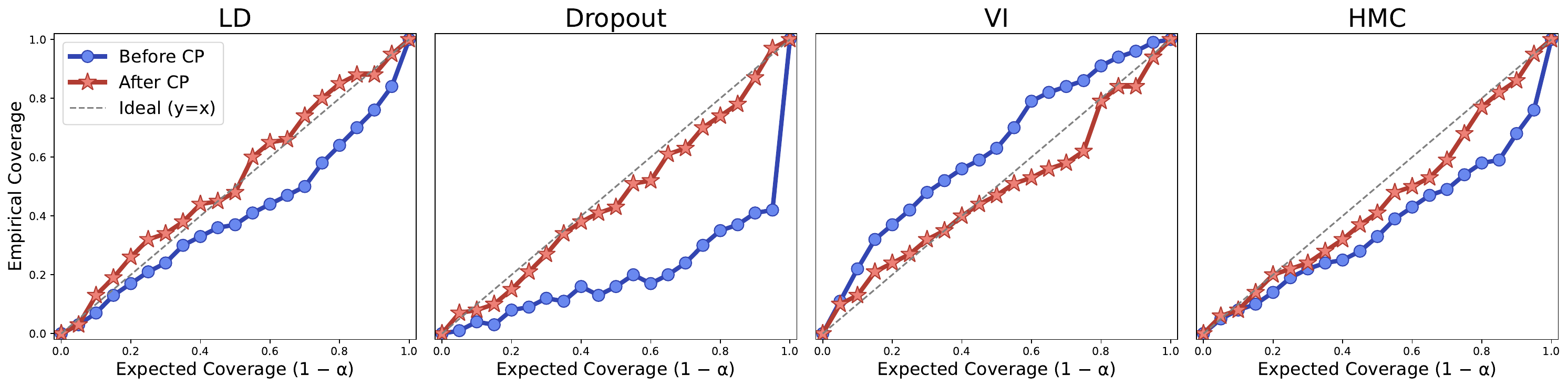}
\caption{Empirical coverage plots for latent-distance PINN, dropout PINN, variational inference B-PINN, and Hamiltonian Monte Carlo B-PINN (left to right).}
  \label{fig:cov_plots}
\end{figure}

Figure~\ref{fig:ac_fd_model} compares the geometric-distance heuristic UQ method (uncalibrated) with its CP-calibrated counterpart. The uncalibrated intervals are overly conservative, while CP significantly sharpens the bands and restore reliable calibration across different significance levels. Figure~\ref{fig:cov_plots} further reports empirical coverage plots for latent-distance, dropout, variational-inference, and Hamiltonian Monte Carlo. Across all cases, CP consistently corrects systematic miscalibration of the heuristic methods, yielding well-calibrated uncertainty intervals—consistent with the improvements observed in the 1D Poisson example.

Table~\ref{tab:allen2d_metrics} reports the performance of different UQ methods on the 2D Allen–Cahn problem at the expected coverage level $1-\alpha=0.95$. Across all methods, the raw models (before CP) display clear miscalibration, with empirical coverage either exceeding or falling short of the target level. After applying CP, the empirical coverage aligns closely with the expected value, and the ACD is significantly reduced. Overall, the results confirm that CP provides systematic and robust calibration for various heuristic UQ methods.

\begin{table}[t]
  \captionsetup{skip=8pt}
  \centering
\caption{Performance metrics for the 2D Allen--Cahn problem at expected coverage level $1-\alpha=0.95$. The ACD is computed over 19 equally spaced $\alpha_k \in [0.05, 0.95]$ on the test dataset.}
  \label{tab:allen2d_metrics}
  \small
  \begin{tabular}{lcccc}
    \toprule
    \multirow{2}{*}{\textbf{Type}} & \multirow{2}{*}{\textbf{Model}} &
    \multicolumn{2}{c}{\textbf{Coverage}} &
    \multirow{2}{*}{\textbf{ACD}} \\
    \cmidrule(lr){3-4}
    & & \textbf{Expected} & \textbf{Empirical} & \\
    \midrule
    \multirow{2}{*}{GD} 
      & Before CP & 0.95 & 1.00  & 0.3395 \\
      & After CP & 0.95 & \textbf{0.96}  & \textbf{0.0329} \\
    \midrule
    \multirow{2}{*}{LD} 
      & Before CP & 0.95 & 0.84 & 0.0929 \\
      & After CP & 0.95 & \textbf{0.95} & \textbf{0.0310} \\
    \midrule
    \multirow{2}{*}{Dropout} 
      & Before CP & 0.95 & 0.42 & 0.2814 \\
      & After CP & 0.95 & \textbf{0.97} & \textbf{0.0386} \\
    \midrule
    \multirow{2}{*}{VI} 
      & Before CP & 0.95 & 0.99 & 0.1205  \\
      & After CP & 0.95 & \textbf{0.94} & \textbf{0.0381}  \\
    \midrule
    \multirow{2}{*}{HMC} 
      & Before CP & 0.95 & 0.76 & 0.1281  \\
      & After CP & 0.95 & \textbf{0.95} & \textbf{0.0486}  \\
    \bottomrule
  \end{tabular}
\end{table}

% \begin{table}[t]
%  \captionsetup{skip=8pt}
%   \centering
%   \caption{Feature-distance model performance on 2D Allen–Cahn at expected coverage level $1-\alpha=0.95$. Average coverage deviation (ACD) is reported as the MAE over the $\alpha$-grid; sharpness is the mean interval width.}
%   \label{tab:ac_fd_metrics}
%   \begin{tabular}{cccccc}
%     \toprule
%     \multirow{2}{*}{\textbf{Model}} &
%     \multicolumn{2}{c}{\textbf{Coverage}} &
%     \multirow{2}{*}{\textbf{Sharpness}} &
%     \multirow{2}{*}{\makecell{\textbf{ACD}}} \\
%     \cmidrule(lr){2-3}
%     & \textbf{Expected} & \textbf{Empirical} & & \\
%     \midrule
%      Heuristic & 0.95 & 1.00 & 0.4572 & 0.2098 \\
%     CP-calibrated   & 0.95 & 0.98 & 0.2990 & 0.0717 \\
%     \bottomrule
%   \end{tabular}
% \end{table}

% Overall, this example complements the 2D Allen--Cahn case: while CP expansion addresses under-coverage in that setting, here it contracts overly wide intervals, achieving both improved calibration and increased sharpness without retraining the underlying PINN.

\subsection{3D Helmholtz Equation}
\label{sec:3dh}

We consider the 3D Helmholtz equation on the unit cube $\Omega=(0,1)^3$ with homogeneous Dirichlet boundary conditions:
\begin{align}
\Delta u(x,y,z) + k^2 u(x,y,z) &= f(x,y,z), && (x,y,z)\in\Omega, \\
u(x,y,z) &= 0, && (x,y,z)\in\partial\Omega,
\end{align}
where the wavenumber is set to $k=\pi$ and the analytical solution $u^*(x,y,z)=\sin(\pi x)\sin(\pi y)\sin(\pi z)$
yields the forcing term $f(x,y,z) = -2\pi^2 \sin(\pi x)\sin(\pi y)\sin(\pi z)$.
% The Dirichlet boundary are employed on all six faces of the cube.
In this example, we adopt the same sampling procedure, but double the sample size from the 2D case, i.e., $600$ for training, $200$ for calibration, and $200$ for testing. We increase the number of interior collocation points to $8{,}000$, while allocate $6{,}144$ boundary points, $1024$ per boundary face.

In this example, we evaluate the performance of three heuristic UQ methods (GD, LD, and Dropout) against their CP-calibrated counterparts. Bayesian variants (VI and HMC) are excluded due to instability in high-dimensional training~\cite{zong2025randomized}. Figure~\ref{fig:helm3d:models:cov} shows that CP consistently corrects both under- and over-coverage, aligning empirical coverage with the ideal $y=1-\alpha$ line across all significance levels. This demonstrates CP’s strong post hoc calibration effect even in higher-dimensional settings.

\begin{figure}[h]
\centering
\includegraphics[width=1.0\linewidth]{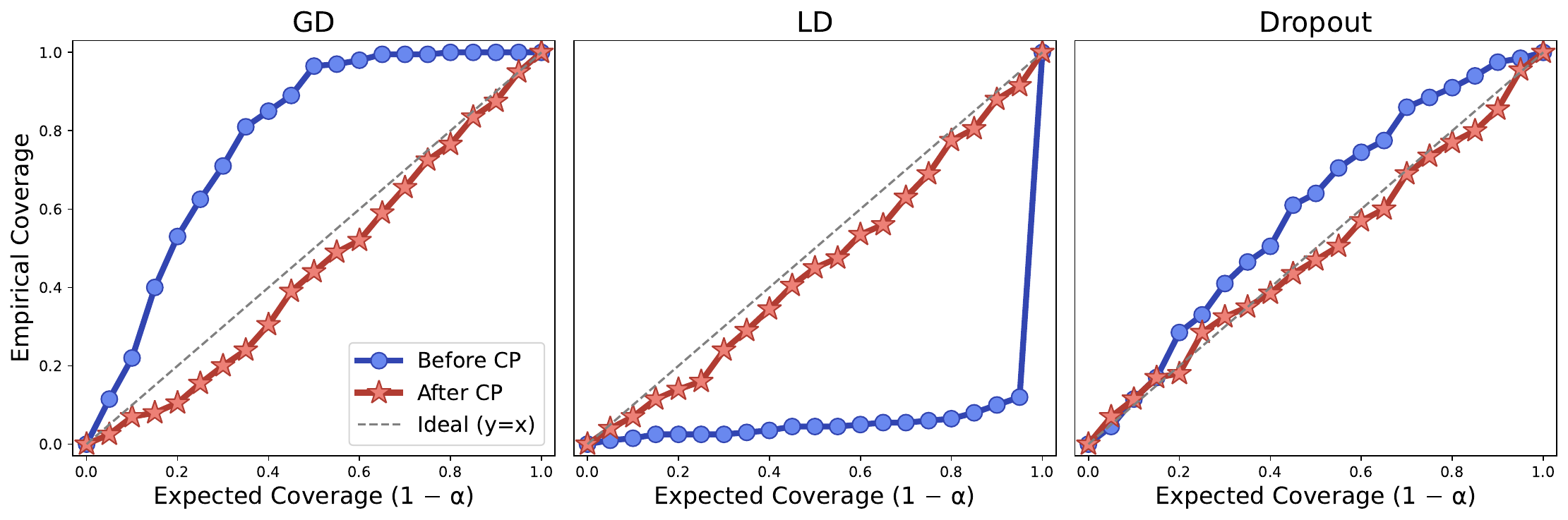}
\caption{Empirical coverage plots for the 3D Helmholtz equation using geometric-distance PINN, latent-distance PINN, and dropout PINN (left to right).}
\label{fig:helm3d:models:cov}
\end{figure}

Table~\ref{tab:helm3d:model:metrics} summarizes performance at the expected coverage level $1-\alpha=0.95$, consistent with the trends observed in Figure~\ref{fig:helm3d:models:cov}. For instance, the latent-distance model achieves only $12\%$ empirical coverage before calibration, indicating a near-complete failure of its uncertainty estimates. After applying CP, coverage improves to $92\%$ and the ACD drops from $0.4090$ to $0.0467$. 
% This highlights the trade-off between calibration and precision inherent in CP for 3D case. 
Even from nearly collapsed baselines (LD in Figure~\ref{fig:helm3d:models:cov}), CP can restores valid coverage across all significance levels without retraining, underscoring its robustness. 
%At the same time, the results clearly illustrate the fundamental trade-off between calibration and sharpness.

\begin{table}[htbp]
  \captionsetup{skip=8pt}
  \centering
\caption{Performance metrics for the 3D Helmholtz problem at expected coverage level $1-\alpha=0.95$. The ACD is computed over 19 equally spaced $\alpha_k \in [0.05, 0.95]$ on the test dataset.}
  \label{tab:helm3d:model:metrics}
  \small
  \begin{tabular}{lcccc}
    \toprule
    \multirow{2}{*}{\textbf{Type}} & \multirow{2}{*}{\textbf{Model}} &
    \multicolumn{2}{c}{\textbf{Coverage}} &
    \multirow{2}{*}{\textbf{ACD}} \\
    \cmidrule(lr){3-4}
    & & \textbf{Expected} & \textbf{Empirical} & \\
    \midrule
    \multirow{2}{*}{GD}
      & Before CP & 0.95 & 1.00  & 0.2643 \\
      & After CP & 0.95 & \textbf{0.95}  & \textbf{0.0517} \\
    \midrule
    \multirow{2}{*}{LD}
      & Before CP & 0.95 & 0.12  & 0.4090 \\
      & After CP & 0.95 & \textbf{0.92} & \textbf{0.0467} \\
    \midrule
    \multirow{2}{*}{Dropout}
      & Before CP & 0.95 & 0.99  & 0.0888 \\
      & After CP & 0.95 & \textbf{0.95} & \textbf{0.0226} \\
    \bottomrule
  \end{tabular}
\end{table}

% \begin{table}[t]
%   \captionsetup{skip=8pt}
%   \centering
%   \caption{Performance of HMC-BPINN on 3D Helmholtz at expected coverage level $1-\alpha=0.95$. Average coverage deviation (ACD) is reported as the MAE over the $\alpha$-grid; sharpness denotes the mean interval width.}
%   \label{tab:helm3d:hmc:metrics}
%   \begin{tabular}{cccccc}
%     \toprule
%     \multirow{2}{*}{\textbf{Model}} &
%     \multicolumn{2}{c}{\textbf{Coverage}} &
%     \multirow{2}{*}{\textbf{Sharpness}} &
%     \multirow{2}{*}{\textbf{ACD}} \\
%     \cmidrule(lr){2-3}
%     & \textbf{Expected} & \textbf{Empirical} & & \\
%     \midrule
%     Heuristic & 0.95 & 0.00 & 0.256 & 0.4524 \\
%     CP-calibrated & 0.95 & 0.97 & 2.487 & 0.0243 \\
%     \bottomrule
%   \end{tabular}
% \end{table}

% This 3D example underscores the practical utility of conformal prediction as a robust post-hoc calibration method. Even starting from an almost completely miscalibrated interval, CP restores empirical coverage across all significance levels without retraining the surrogate model. The results also exemplify the fundamental balance between calibration quality and interval sharpness: improved coverage often comes at the cost of wider prediction intervals.

\section{Extension}
\label{sec:extension}

In this section, we extend the CP--PINN framework to incorporate local adaptivity in UQ for PDEs. When noise or model error exhibits strong spatial heterogeneity~\cite{li2019spatial}, a single global scaling factor $q$ (cf.~Section~\ref{sec:sub:cp}) may yield intervals that are overly conservative in some regions and under-confident in others. To overcome this limitation, we propose a local CP method (Section~\ref{sec:sub:local_cp}) with an efficient algorithmic implementation (Algorithm~\ref{alg:local_q_scaled_cp}). The approach requires no extra data or model retraining, preserves finite-sample coverage (Theorem~\ref{thm:local_cp}), and achieves pointwise adaptivity by adjusting interval widths to local uncertainty patterns. Its effectiveness is demonstrated through numerical experiments in Section~\ref{sec:sub:local_numerics}.

% In this section, we extend the CP--PINN framework to incorporate local adaptivity in UQ for PDEs. This extension is motivated by settings where noise or model error exhibits strong spatial heterogeneity~\cite{li2019spatial, garrigues2006quantifying}. In such cases, CP (cf.~Section~\ref{sec:sub:cp}), which relies on a single global scaling factor $q$, may lead to prediction intervals that are overly conservative in some regions and under-confident in others. To address this limitation, we propose a local CP method in Section~\ref{sec:sub:local_cp}, together with its algorithmic implementation, presented in Algorithm~\ref{alg:local_q_scaled_cp}. The approach requires neither additional data nor model retraining, and preserves the finite-sample marginal coverage guarantees of CP (cf.~Theorem~\ref{thm:local_cp}), and achieves point-wise adaptivity by modulating interval widths according to local uncertainty patterns. The effectiveness of the method is evaluated through numerical experiments, provided in Section~\ref{sec:sub:local_numerics}.

\subsection{Local Conformal Prediction (Local CP)}
\label{sec:sub:local_cp}

Recall the CP setting from Section~\ref{sec:sub:cp}. Instead of computing residual-based conformity scores solely on the calibration set, we exploit the full training data $\mathcal D_{\mathrm{data}} = \{(x_i,u_i)\}_{i=1}^{N_{\mathrm d}}$,
and define normalized residuals
\[
\mathcal S_{\mathrm{data}} = \Bigg\{ s_i : s_i = \frac{|u_i - u_\theta(x_i)|}{\sigma(x_i)} \Bigg\}_{i=1}^{N_{\mathrm d}},
\]
where $u_\theta$ is the trained PINN surrogate and $\sigma$ is a baseline heuristic estimator (cf~Section~\ref{sec:sub:original_uq}).

The goal is to learn a smooth, input-dependent conditional quantile function $g_\phi$ parameterized by $\phi$, that approximates the $(1-\alpha)$-quantile of the conditional distribution. To achieve this, we minimize the empirical \emph{pinball loss} (a convex surrogate for quantile regression)~\cite{chung2021beyond}:
\begin{equation}
\mathcal L(\phi)
=\frac{1}{N_{\mathrm{d}}}\sum_{i=1}^{N_{\mathrm{d}}}
\rho_{1-\alpha}\bigl(s_i - g_\phi(x_i)\bigr),
\qquad 
\rho_{1-\alpha}(t) = (1-\alpha)\,t_+ + \alpha\,(-t)_+,
\label{eq:pinnball}
\end{equation}
where $t_+=\max(t,0)$ and $(-t)_+=\max(-t,0)$ denote the positive and negative parts, respectively. This loss enforces that $g_\phi(x)$ upper-bounds the residuals $s_i$ at approximately the $(1-\alpha)$ conditional quantile level.

Compared with the global quantile $q^{\mathrm{cp}}_{1-\alpha}$ used in CP, 
the learned function $g_\phi(x)$ provides an $x$-dependent estimator of the conditional quantile. 
Specifically, let $\mathcal D_{\mathrm{cal}}=\{(x_j,u_j)\}_{j=1}^{N_{\mathrm{c}}}$ 
be the calibration dataset, and define the localized calibration scores as
\begin{equation}
\ell_j=\frac{|u_j-u_{\theta}(x_j)|}{g_\phi(x_j)\,\sigma(x_j)},\qquad j=1,\dots,N_{\mathrm{c}}.
\label{eq:local_cali}
\end{equation}
Let $\ell^*$ denote the empirical $(1-\alpha)$–quantile of the calibration scores 
$\{\ell_j\}_{j=1}^{N_{\mathrm{c}}}$, i.e.,
\[
\ell^* := \inf\left\{\, t \in \mathbb{R} : \frac{1}{N_{\mathrm{c}}}\sum_{j=1}^{N_{\mathrm{c}}}\mathbf 1\{\ell_j \le t\} \ge 1-\alpha \right\}.
\]
We then define the local quantile estimator as $q_\phi(x):=\ell^*\cdot g_\phi(x)$,
and obtain the final prediction interval for a new input $x_{\rm new}$
\begin{equation}
I^{\mathrm{scp}}_{1-\alpha}(x_{\rm new})
=\bigl[u_{\theta}(x_{\rm new})-q_\phi(x_{\rm new})\,\sigma(x_{\rm new}),\;\; u_{\theta}(x_{\rm new})+q_\phi(x_{\rm new})\,\sigma(x_{\rm new})\bigr].
\label{eq:local_cp_interval}
\end{equation}

The detailed algorithm is presented in Algorithm~\ref{alg:local_q_scaled_cp}. 
% By incorporating an $x$-dependent quantile estimator, local CP maintains the property of finite-sample marginal coverage guarantees from conformal prediction, while enhancing point-wise adaptivity without requiring additional data or model retraining. 
The following Theorem presents the finite-sample coverage guarantee of local CP. For brevity, we provide only a proof sketch, since the full argument involves technical details that fall outside the scope of this paper and will be addressed rigorously elsewhere.

\begin{theorem}[Finite-sample coverage guarantee of local CP]
\label{thm:local_cp}
Fix $\alpha\in(0,1)$. Given independent i.i.d.\ samples $\mathcal D_{\mathrm{data}},\mathcal D_{\mathrm{cal}}$ from $(\mathcal{X},\mathcal{U})$, let $I^{\mathrm{lcp}}_{1-\alpha}$ be the Local-CP interval defined in~\eqref{eq:local_cp_interval}. Then for any independent $(x_{\mathrm{new}},u_{\mathrm{new}})\sim(\mathcal{X},\mathcal{U})$,
\[
\mathbb P\!\left( u_{\mathrm{new}} \in I^{\mathrm{lcp}}_{1-\alpha}(x_{\mathrm{new}}) \right) \ge 1-\alpha .
\]
Moreover, if the calibration scores in \eqref{eq:local_cali} have a continuous distribution, the coverage holds exactly modulo the standard $1/(N_{\mathrm{c}}+1)$ finite-sample correction:
\[
\mathbb P\!\left( u_{\mathrm{new}} \in I^{\mathrm{lscp}}_{1-\alpha}(x_{\mathrm{new}}) \right) 
= \frac{k}{N_{\mathrm{c}}+1}
\in\Bigl[\,1-\alpha,\;1-\alpha+\frac{1}{N_{\mathrm{c}}+1}\,\Bigr),
\]
where $k=\lceil (N_{\mathrm{c}}+1)(1-\alpha)\rceil$.
\end{theorem}

\begin{proof}[Sketch of proof]
Assume $\sigma:\mathcal X\to(0,\infty)$ and $g_\phi:\mathcal X\to(0,\infty)$ are measurable and strictly positive.
Define the scaled residual scores, for $j=1,\dots,N_{\mathrm c}$,
\[
\ell_j = \frac{\bigl|u_j-u_{\theta}(x_j)\bigr|}{g_\phi(x_j)\,\sigma(x_j)},
\qquad
\ell_{\mathrm{new}} = \frac{\bigl|u_{\mathrm{new}}-u_{\theta}(x_{\mathrm{new}})\bigr|}{g_\phi(x_{\mathrm{new}})\,\sigma(x_{\mathrm{new}})}.
\]
Let $\ell_{(1)}\le \cdots \le \ell_{(N_{\mathrm c})}$ denote the order statistics of $\{\ell_j\}_{j=1}^{N_{\mathrm c}}$, and set $k=\lceil (N_{\mathrm c}+1)(1-\alpha)\rceil$.

Conditioning on $\mathcal D_{\mathrm{data}}$, $u_\theta$, $\sigma$, and $g_\phi$ are deterministic and independent of $\mathcal D_{\mathrm{cal}}$ and $(x_{\mathrm{new}},u_{\mathrm{new}})$. 
Since $(x_j,u_j)$ and $(x_{\mathrm{new}},u_{\mathrm{new}})$ are i.i.d., it follows that the random variables $(\ell_1,\dots,\ell_{N_{\mathrm c}},\ell_{\mathrm{new}})$ are exchangeable (conditionally on $\mathcal D_{\mathrm{data}}$). The standard conformal rank argument then yields
\[
\mathbb P\Bigl(\,\ell_{\mathrm{new}}\le \ell_{(k)} \Big| \mathcal D_{\mathrm{data}} \Bigr)
\ge\frac{k}{N_{\mathrm c}+1}
\ge1-\alpha,
\]
where the first inequality becomes equality if the distribution of the scores is continuous (no ties). 
By the definition of the Local-CP interval with $q_\phi(x)=\ell_{(k)}\,g_\phi(x)$, we have
\[
\{\ell_{\mathrm{new}}\le \ell_{(k)}\}
\;\Longleftrightarrow\;
\bigl|u_{\mathrm{new}}-u_\theta(x_{\mathrm{new}})\bigr|
\le
\ell_{(k)}\,g_\phi(x_{\mathrm{new}})\,\sigma(x_{\mathrm{new}})
\;\Longleftrightarrow\;
u_{\mathrm{new}}\in I^{\mathrm{lcp}}_{1-\alpha}(x_{\mathrm{new}}).
\]
Taking expectations with respect to $\mathcal D_{\mathrm{data}}$ gives
\[
\mathbb P\bigl(u_{\mathrm{new}}\in I^{\mathrm{lcp}}_{1-\alpha}(x_{\mathrm{new}})\bigr)
\;\ge\; 1-\alpha.
\]
If the score distribution is continuous, then
\[
\mathbb P\Bigl(\,\ell_{\mathrm{new}}\le \ell_{(k)}\Big|\mathcal D_{\mathrm{data}} \Bigr)=\frac{k}{N_{\mathrm c}+1},~~\text{hence}~~
\mathbb P\bigl(u_{\mathrm{new}}\in I^{\mathrm{lcp}}_{1-\alpha}(x_{\mathrm{new}})\bigr)
=\frac{k}{N_{\mathrm c}+1}
\in \Bigl[1-\alpha,\,1-\alpha+\tfrac{1}{N_{\mathrm c}+1}\,\Bigr),
\]
which yields the stated results.
\end{proof}

This theorem establishes that local CP inherits the finite-sample marginal guarantees of standard conformal prediction, while providing $x$-dependent intervals that adapt to spatial variations in uncertainty. Numerical validation is presented in the next subsection.

\begin{algorithm}[t]
\caption{Local Conformal Prediction (Local CP)}
\label{alg:local_q_scaled_cp}
\begin{algorithmic}[1]
\Statex \textbf{Input:} Training data $\mathcal D_{\mathrm{data}}=\{(x_i,u_i)\}_{i=1}^{N_{\mathrm{d}}}$, calibration data $\mathcal D_{\mathrm{cal}}=\{(x_j,u_j)\}_{j=1}^{N_{\mathrm{c}}}$, miscoverage level $\alpha\in(0,1)$, baseline model $(u_{\theta},\sigma)$.
\vspace{2pt}
\Statex \textbf{I. Fit conditional quantile predictor.}
\For{$i=1,\dots,N_{\mathrm{d}}$}
  \State Compute conformity score: $s_i \gets \dfrac{|u_i - u_{\theta}(x_i)|}{\sigma(x_i)}$.
\EndFor
\State Train $g_{\phi}:\mathcal X\to\mathbb R_{>0}$ by minimizing the pinball loss at level $1-\alpha$ on $\{(x_i,s_i)\}_{i=1}^{N_{\mathrm{d}}}$ by~\eqref{eq:pinnball}.

\vspace{2pt}
\Statex \textbf{II. Calibration of multiplicative factor.}
\For{$j=1,\dots,N_{\mathrm{c}}$}
  \State Compute calibration score:
  $\ell_j \gets \dfrac{|u_j-u_{\theta}(x_j)|}{g_{\phi}(x_j)\,\sigma(x_j)}$.
\EndFor
\State Let $\ell_{(1)}\le\cdots\le \ell_{(N_{\mathrm{c}})}$ be the order statistics of $\{\ell_j\}$.
\State Set $k \gets \lceil (N_{\mathrm{c}}+1)(1-\alpha)\rceil$ and $\ell^*\gets \ell_{(k)}$.

\vspace{2pt}
\Statex \textbf{III. Prediction at a new test point $x_{\mathrm{new}}$.}
\State Define local quantile estimate $q_{\phi}(x_{\mathrm{new}})=\ell^*\cdot g_{\phi}(x_{\mathrm{new}})$.
\State Construct prediction interval
\[
I^{\mathrm{lcp}}_{1-\alpha}(x_{\mathrm{new}})
=\bigl[u_{\theta}(x_{\mathrm{new}})-q_{\phi}(x_{\mathrm{new}})\,\sigma(x_{\mathrm{new}}),\;\;
      u_{\theta}(x_{\mathrm{new}})+q_{\phi}(x_{\mathrm{new}})\,\sigma(x_{\mathrm{new}})\bigr].
\]
\State \Return $I^{\mathrm{lcp}}_{1-\alpha}(x_{\mathrm{new}})$.
\end{algorithmic}
\end{algorithm}

\subsection{Numerical Results}
\label{sec:sub:local_numerics}

Before presenting numerical results, we introduce \emph{sharpness}, a commonly used metric for assessing the informativeness of prediction intervals~\cite{hu2022robust}. It is defined as the average interval width:
\begin{equation}
\text{Sharpness} := \frac{1}{N_{\rm test}} \sum_{i=1}^{N_{\rm test}} \big(U_i - L_i\big).
\label{eq:sharpness}
\end{equation}
Lower values indicate narrower, more informative intervals. However, sharpness is meaningful only when models are equally calibrated, since overly narrow intervals under poor calibration may simply reflect under-coverage rather than genuine confidence~\cite{hu2022robust}.

% In each case, we compare the error estimates produced by the heuristic baseline, standard CP, and local CP against the underlying heteroskedastic data noise~\cite{matthews1991dynamics}.

\subsubsection{1D Damped Harmonic Oscillator Equation}
\label{sec:sub:sub:1dosci}

We consider the one-dimensional damped harmonic oscillator defined on the interval $T=(0,5)$:
\begin{equation}
u''(t) + 2\zeta\omega u'(t) + \omega^2 u(t) = f(t), 
\qquad t \in T,
\end{equation}
subject to initial conditions $u(0)=u_0$ and $u'(0)=v_0$. The oscillator parameters are fixed as $\omega = 2\pi$ and damping ratio $\zeta=0.05$. For $f(x)=0$, the corresponding analytical solution is
\[
u^*(t) = e^{-\zeta\omega t}\!\left(
u_0 \cos(\tilde{\omega} t) + \frac{v_0 + \zeta \omega u_0}{\tilde{\omega}} \sin(\tilde{\omega} t)
\right),
\qquad 
\tilde{\omega}=\omega\sqrt{1-\zeta^2}.
\]

% To incorporate the physics into training, we penalize the residual
% \[
% \mathcal{R}(u)(t) = u''(t) + 2\zeta\omega\,u'(t) + \omega^2 u(t),
% \]
% together with initial-condition losses at $t=0$. This enforces consistency with both the ODE dynamics and prescribed initial conditions.

We generate $300$ training samples and $150$ calibration samples from the analytical solution $u^*(t)$. Another $1{,}000$ test samples are drawn to thoroughly test the method's robustness for a downstream modeling task. To introduce heteroskedasticity, first, an i.i.d. Gaussian perturbation with standard deviation $\sigma=0.05$ is added to the simulated $u^*$.
\[
\tilde{u}_i = u^*(t_i) + \sigma\,\varepsilon_i, 
\qquad \varepsilon_i \sim \mathcal N(0,1).
\]
The perturbed observations are further corrupted with additive Gaussian noise with location-dependent variance. Specifically, the point-wise noise level is:
\begin{align}
    &\sigma_{\mathrm{het}}(t) = \sum_{r=1}^3 b \exp\Bigl(-\frac12\Bigl(\frac{t-c_r}{w_r}\Bigr)^2\Bigr)\mathbf 1\big\{|t-c_r|\le w_r\big\}, \nonumber \\
    &\text{where } (c_r,w_r) \in \{(1.0,\,0.2),\;(2,\,0.2),\;(3,\,0.2)\},
    \label{eqn:local_cp_1D}
\end{align}
where we define the noise bump as $b=0.3$. Noisy samples are then drawn as $u_i = \tilde{u}_i + \sigma_{\mathrm{het}}(t_i)\,z_i,\ z_i \sim \mathcal N(0,1)$, thereby capturing the intrinsic variability.

% We place $200$ collocation points uniformly across $T$ to enforce the residual and incorporate the initial conditions $u(0)=1$ and $u'(0)=0$. The model is optimized using Adam (learning rate $10^{-3}$) for $20{,}000$ epochs, with loss weights set as $\lambda_{\mathrm{pde}}=1.0$, $\lambda_{\mathrm{ic}}=10.0$, and $\lambda_{\mathrm{data}}=3.0$.

% The conditional quantile predictor $g_{\phi,1-\alpha}$ employs the same architecture and is trained for $20{,}000$ epochs with learning rate $10^{-4}$, minimizing the pinball loss at level $1-\alpha$.

Figure~\ref{fig:lscp:modelplot} compares prediction intervals from standard CP and local CP at at significance level $\alpha=0.1$. The standard CP, relying on a single global quantile, fails to capture domain-varying noise: intervals are overly wide in smooth regions yet too narrow over the noisy ``islands" (cf.~\eqref{eqn:local_cp_1D}), causing local undercoverage. In contrast, local CP employs an $x$-dependent quantile that adapts interval widths to local variability, yielding tight bands in stable regions and wider bands in noisy areas, thus aligning more closely with the true heteroskedastic structure.

\begin{figure}[h]
\centering
\includegraphics[width=0.95\linewidth]{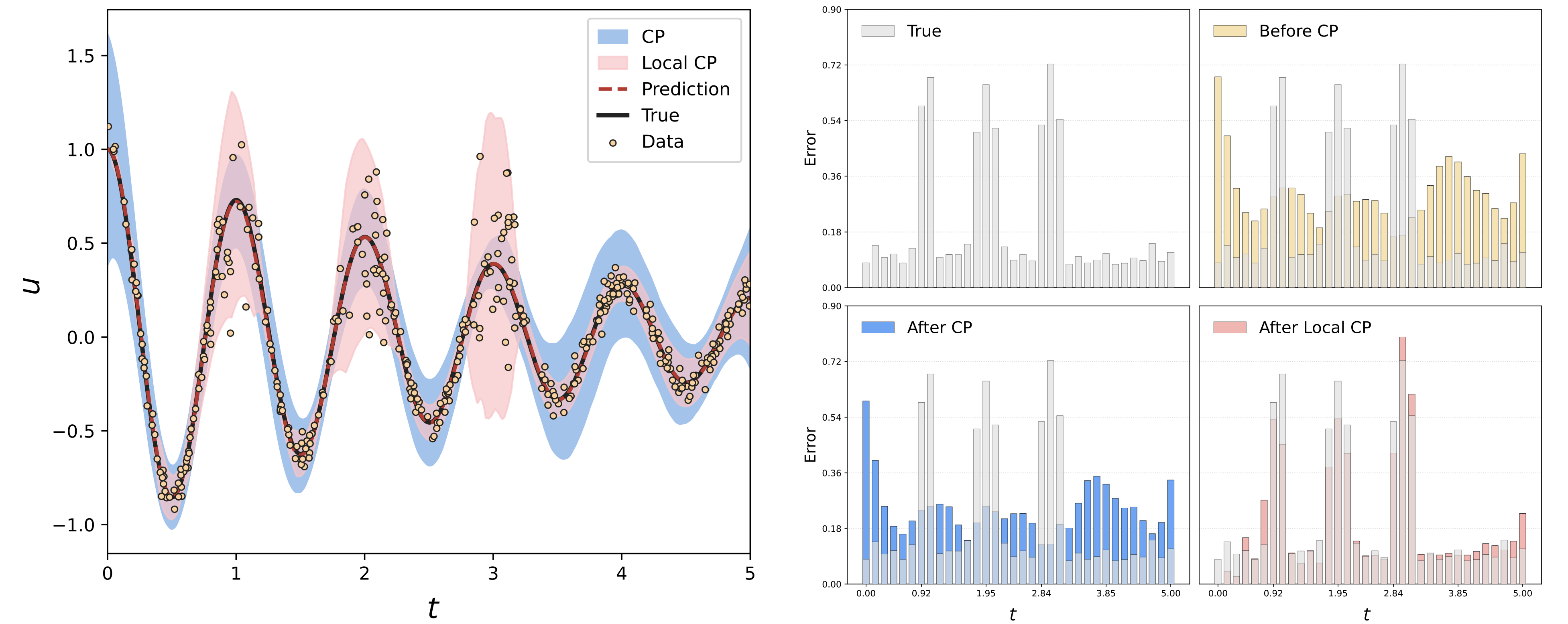}
\caption{Comparison of CP and local CP at $\alpha=0.1$. \textbf{Left:} prediction intervals with CP (blue) and local CP (red). \textbf{Right:} The absolute-error distributions, the predictive interval width before CP, after CP, and after local CP, showing that local CP achieves superior calibration.}
\label{fig:lscp:modelplot}
\end{figure}

Quantitative results are reported in Table~\ref{tab:local:metrics}, 
with empirical coverage curves for different evaluation regions shown in Figure~\ref{fig:lscp:covplot}. We see that although both CP and local CP achieve good global coverage, their performance differ markedly in regions of elevated noise. At the expected coverage level $1-\alpha=0.95$, the standard CP attains only $0.76$ empirical coverage across the three high-noise islands, whereas local CP achieves $0.94$, closely aligning with the target level of $0.95$. Importantly, this improvement in local coverage is achieved without sacrificing sharpness. Over the entire domain, local CP attains a sharpness of $0.66$, significantly lower than that of the standard CP ($0.9476$), thereby demonstrating its ability to produces intervals that are both more calibrated and more informative than those of CP. Sharpness values for the local regions of elevated noise are omitted, as the standard CP model is not calibrated in these regions and therefore not directly comparable.

\begin{figure}[t]
\centering
\includegraphics[width=0.90\linewidth]{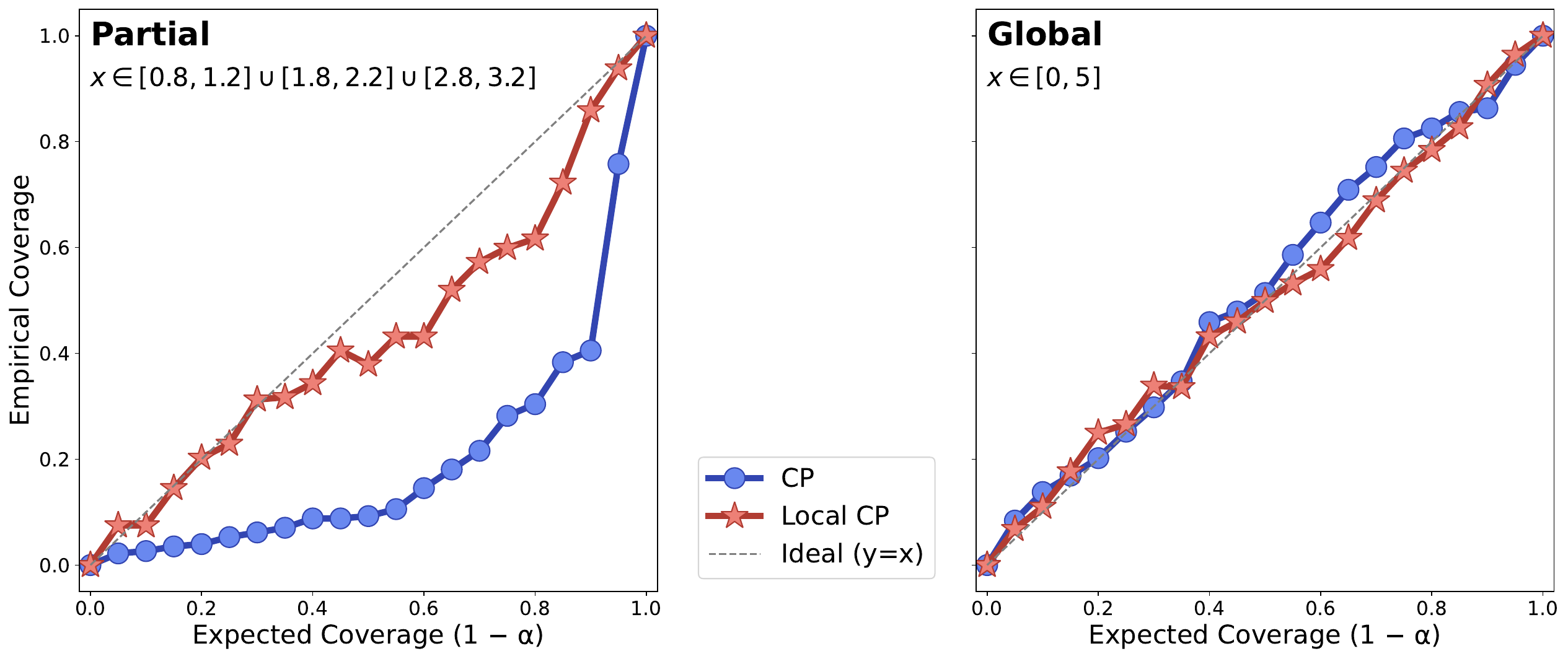}
\caption{Empirical coverage plots on the three noisy regions \textbf{(left)} and on the global domain \textbf{(right)}. Local CP achieves closer alignment with the expected $1-\alpha$ line, indicating improved calibration compared to standard CP.}
\label{fig:lscp:covplot}
\end{figure}

The coverage curves in Figure~\ref{fig:lscp:covplot} further underscore these findings. Restricting the evaluation to the heteroskedastic regions, the CP coverage curve consistently falls below the $y=x$ diagonal for all values of $\alpha$, indicating a systematic deficiency in quantifying heteroskedastic points. By contrast, local CP adheres to the diagonal, achieving a substantially smaller ACD of $0.0669$, which is more than a fourfold improvement over standard CP, whose ACD is much larger at $0.2925$. This result provides further evidence that local CP systematically outperforms CP, particularly in regions characterized by elevated noise.

\begin{table}[t]
  \captionsetup{skip=8pt}
  \centering
\caption{Performance metrics for CP and local CP at $\alpha=0.05$ are reported both partially (within heteroskedastic regions) and globally (over the full domain). Sharpness for the partial regions is omitted, since CP is uncalibrated there and thus not comparable. ACD is evaluated over 19 equally spaced $\alpha_k \in [0.05, 0.95]$ on the test set.}
  \label{tab:local:metrics}
  \small
  \begin{tabular}{cccccc}
    \toprule
    \multirow{2}{*}{\textbf{Test Regions}} & 
    \multirow{2}{*}{\textbf{Model}} &
    \multicolumn{2}{c}{\textbf{Coverage}} &
    \multirow{2}{*}{\textbf{Sharpness}} &
    \multirow{2}{*}{\textbf{ACD}} \\
    \cmidrule(lr){3-4}
    & & \textbf{Expected} & \textbf{Empirical} & & \\
    \midrule
    \multirow{2}{*}{Partial}
      & CP  & 0.95 & 0.76 & - & 0.2925 \\
      & Local CP & 0.95 & \textbf{0.94} & - & \textbf{0.0669} \\
    \midrule
    \multirow{2}{*}{Global}
      & CP  & 0.95 & 0.94 & 0.9476 & 0.0250 \\
      & Local CP & 0.95 & 0.96 & \textbf{0.6600} & \textbf{0.0183} \\
    \bottomrule
  \end{tabular}
\end{table}

\subsubsection{2D Allen-Cahn Equation}
\label{sec:sub:sub:para:ac2d}

To further demonstrate the adaptiveness of the local conformal prediction, we adopt the same data generation and Geometric-distance PINN training procedures as described in Section~\ref{sec:ac2d}, and similarly increase the test set size to $2{,}000$ to simulate an industrial scenario to test method's reliability in higher dimensional input space. Building upon the previous global noise setting, we introduce irregularly shaped regions into the two-dimensional input space, within which all points are perturbed by higher noise. To better reflect realistic conditions, we smooth the region boundaries using a sigmoid transformation, ensuring that the noise level decays gradually. 
% Implementation details are omitted here for clarity. 

Figure~\ref{fig:lscp:2dac} compares the UQ intervals generated by the Geometric-distance baseline, CP, and local CP, respectively, illustrating the flexibility of local CP when confronted with irregular noise patterns. While CP improves calibration of the uncertainty intervals compared to the uncalibrated baseline, the use of a single global conformity score quantile fails to adapt to spatially heterogeneous noise, leading to locally under-covered and over-covered, for irregular and smooth regions respectively. In contrast, local CP effectively captures location-sensitive noise patterns across all three scenarios. It adaptively expands interval widths where needed, while preserving narrow, unperturbed intervals in smoother regions. This finding is consistent with the 1D example in Section~\ref{sec:sub:sub:1dosci}, where local CP effectively captured spatially varying uncertainty. We observe from Figure~\ref{fig:lscp:2dac_cov} that both CP and local CP achieve valid coverage when evaluated over the full domain, consistent with their finite-sample guarantees (cf.~Theorem~\ref{thm:cp} and Theorem~\ref{thm:local_cp}). The advantage of local CP is most pronounced at the local scale: by adapting interval widths to spatial heterogeneity, it corrects the under-coverage in high-variance regions and the over-coverage in smooth regions that persist under a single global quantile. Consequently, improvements in global metrics appear more modest, since averaging over the domain tends to mask spatial variability.

% \yf{
% \begin{remark}
%     Over the entire domain, the discrepancy between the two methods is less evident as the regions where CP under-covers (in elevated-noise areas) are counterbalanced by regions where it over-covers (in smooth areas), giving a misleading impression of adequate global coverage performance. In contrast, local CP accurately identifies the heterogeneous regions. Thus, while the two methods may seem comparable under global metrics, the superior performance of local CP becomes evident once the underlying noise structure is taken into account.
% \end{remark}
% } 
% In conclusion, our numerical results underscore the practical value of local CP for scientific uncertainty quantification. By adapting to local noise structures, local CP produces better-calibrated and more informative predictive intervals than standard CP. Future work may extend this framework to more complex and higher-dimensional PDE systems or other scientific machine learning frameworks.

\begin{figure}[h]
\centering
\includegraphics[width=0.95\linewidth]{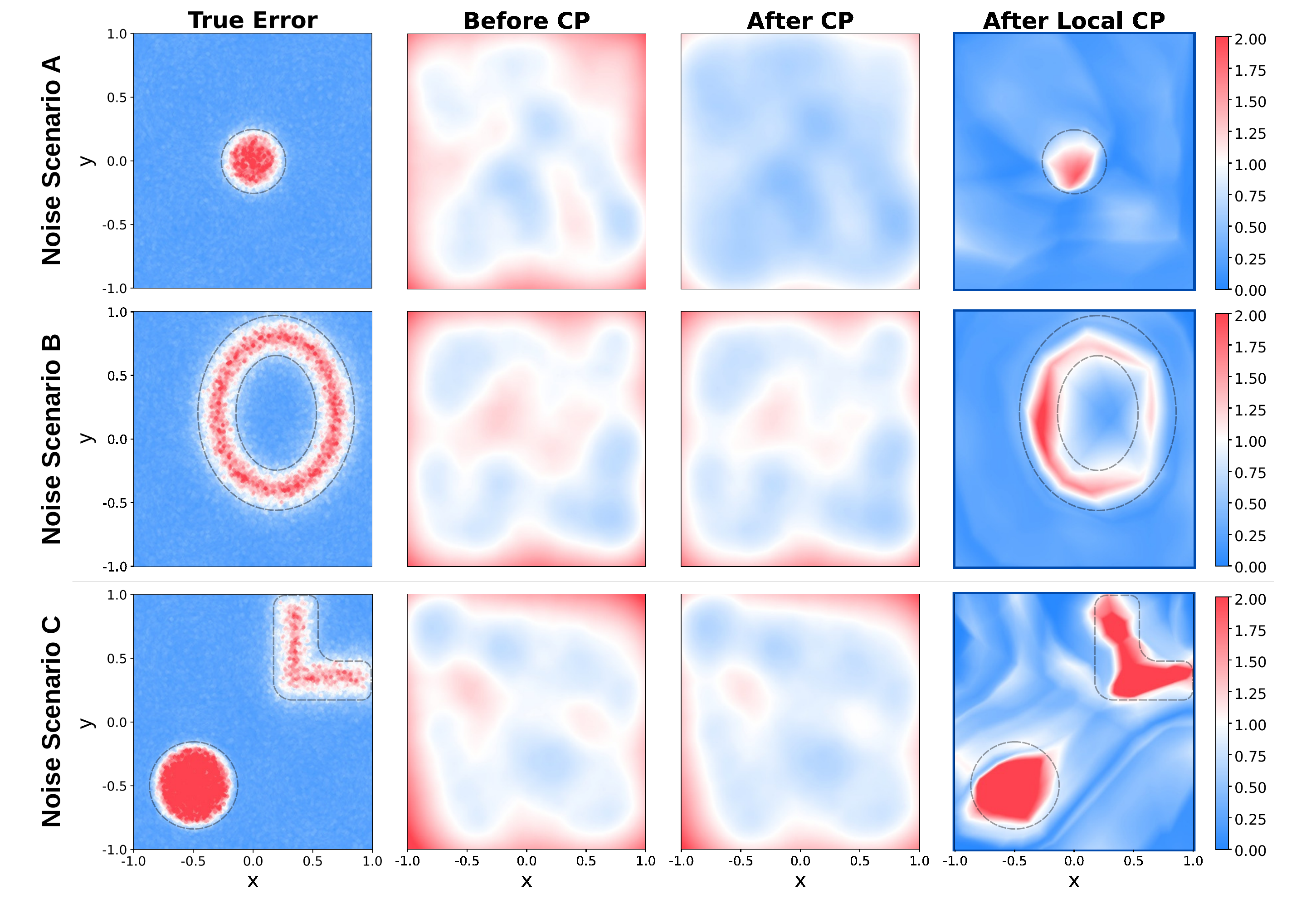}
\caption{Geometric-distance PINN uncertainty calibration for the 2D Allen–Cahn equation under different heteroskedastic noise patterns. \textbf{Row:} Different noise patterns. \textbf{Column:} The true absolute-error distributions and the interval widths for the baseline model, CP, and local CP at $\alpha=0.05$, respectively, from the first column to the fourth column. The irregular noisy regions are distinguished with dashed lines.} 
\label{fig:lscp:2dac}
\end{figure}

\begin{figure}[h]
\centering
\includegraphics[width=1.0\linewidth]{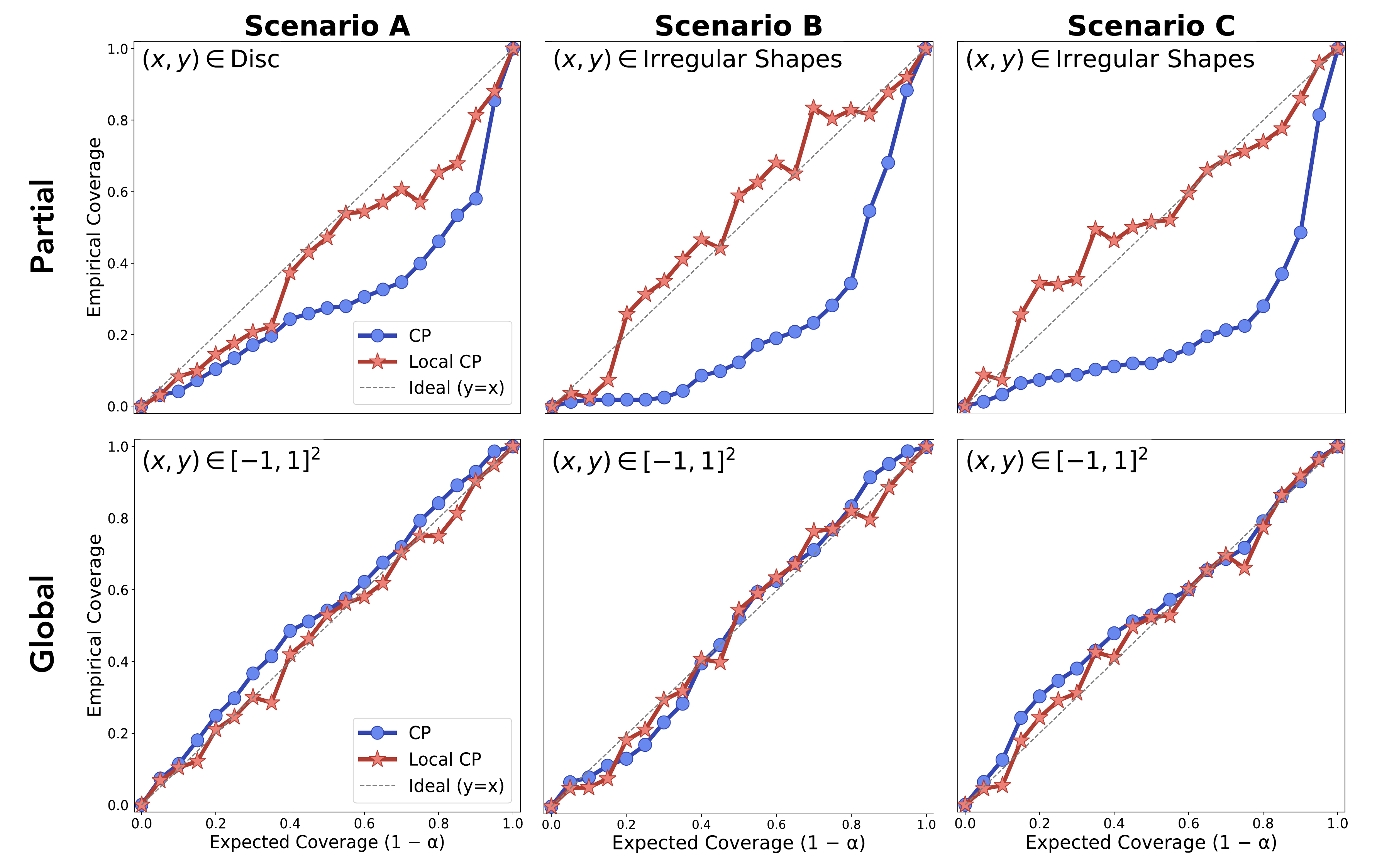}
\caption{Empirical coverage curves of the Geometric-distance PINN across three local noise scenarios. \textbf{Top:} results restricted to noise-elevated sub-regions. \textbf{Bottom:} results over the full domain $(x,y)\in[-1,1]^2$. Each panel compares standard conformal prediction (CP, blue) and local conformal prediction (local CP, red) against the expected = empirical reference (dashed).}
\label{fig:lscp:2dac_cov}
\end{figure}

% Heteroscedastic soft-step noise “islands” (mathematical specification)

\section{Conclusion and Outlook}
\label{sec:conclusion}

In this work, we propose a CP framework for calibrating the heuristics UQ in PINNs. Unlike heuristic or Bayesian approaches, the method is distribution-free and yields prediction intervals with rigorous finite-sample coverage guarantees. By incorporating local quantile estimation, the framework achieves spatially adaptive uncertainty quantification. In addition, we systematically evaluate a range of heuristic UQ methods and metrics, providing a comprehensive assessment that underscores the robustness of our framework. Numerical experiments on benchmark PDEs, including Poisson, Allen–Cahn, and Helmholtz equations, demonstrate that the proposed method consistently delivers reliable calibration. These results highlight the potential of CP to bridge deterministic PINNs with statistically principled UQ, thereby advancing the reliability and trustworthiness of PINN-based scientific computing.

Promising directions for future work include extensions to inverse and partially observed problems, time-dependent PDEs, and integration with operator-learning paradigms such as DeepONets and Fourier Neural Operators. In addition, the flexibility of the framework allows for alternative choices of nonconformity scores, such as energy-norm residuals, multi-output joint measures, hybrid residual–variance scores, or learned data-driven variants, thereby enabling task-specific and adaptive calibration.

\appendix

\section{Supplementary Details for Bayesian PINNs}
\label{sec:apd:vi}

\subsection{Variational Inference (VI)}
\label{sec:sub:variational_inference}

Variational inference (VI) approximates the generally intractable posterior distribution by a tractable parametric family through an optimization task. Assuming a fully factorized Gaussian surrogate posterior, we write
\begin{equation}
  q_{\phi}(\theta)
  \;=\;
  \prod_{j=1}^{d_\theta}
  \mathcal N\bigl(\theta_j | \mu_j,\sigma_j^{2}\bigr),
  \quad
  \sigma_j = \operatorname{softplus}(\rho_j),
  % \label{eq:mf_gaussian}
\end{equation}
where $\phi = \{(\mu_j,\rho_j)\}_{j=1}^{d_\theta}$ are the variational parameters, and the softplus transformation guarantees strictly positive standard deviations~\cite{blundell_weight_2015}.  

With this surrogate distribution, VI converts Bayesian inference into the minimization of the negative evidence lower bound (ELBO):
\begin{align}
  \min_{\phi}- \mathcal L_{\mathrm{ELBO}}(\phi) 
  &:= -
     \underbrace{\mathbb E_{q_{\phi}}
       \bigl[\log p(\mathcal D|\theta)\bigr]}_{\text{expected log–likelihood}}
     +
     \underbrace{\mathrm{KL}\bigl(q_{\phi}(\theta)||p_{0}(\theta)\bigr)}_{\text{complexity penalty}}.
  % \label{eq:elbo}
\end{align}

The first term, the expected log-likelihood, is typically approximated via Monte Carlo sampling of $\theta \sim q_{\phi}(\theta)$. The second term, the KL divergence, admits a closed form when both $q_{\phi}(\theta)$ and the prior $p_0(\theta)$ are Gaussian. For a single parameter dimension:
\begin{equation}
  \mathrm{KL}
  \bigl(
    q_{\phi}(\theta)||p_{0}(\theta)
  \bigr)
  =
  \log\frac{\sigma_{0}}{\sigma}
  +\frac{\sigma^{2}+\mu^{2}}{2\sigma_{0}^{2}}
  -\tfrac12,
  \quad \text{where } ~~
  \begin{aligned}
    q_{\phi}(\theta) &= \mathcal{N}(\mu, \sigma^2), \\
    p_{0}(\theta) &= \mathcal{N}(0, \sigma_0^2).
  \end{aligned}
  % \label{eq:kl_gaussians}
\end{equation}

During training, we employ the reparameterization trick
\(
\theta = \mu + \sigma \varepsilon,~
\varepsilon \sim \mathcal{N}(0, I),
\)
yielding unbiased, low-variance gradient estimates of $\nabla_{\!\phi}\mathcal L_{\mathrm{ELBO}}$ through standard backpropagation. To accelerate training, we use mini-batches $\mathcal B \subset \mathcal D$, replacing~\eqref{eq:elbo} with
\begin{align}
  \min_{\phi}\ - \mathcal L_{\mathrm{ELBO}}(\phi) 
  &=
  -\underbrace{\mathbb E_{q_{\phi}}
       \!\bigl[\log p(\mathcal B|\theta)\bigr]}_{\text{expected log–likelihood}}
     + \underbrace{\mathrm{KL}\bigl(q_{\phi}(\theta)||p_{0}(\theta)\bigr)}_{\text{complexity penalty}},
\end{align}
and optimize using the Adam algorithm with typically one Monte Carlo sample per batch~\cite{yang_b-pinns_2021}.

\subsection{Hamiltonian Monte Carlo (HMC)}
\label{sec:apd:hmc}

Consider the posterior
\begin{equation}
    p(\theta|\mathcal D)
    \propto p(\mathcal D|\theta)p_0(\theta)
    = \exp\bigl(-U(\theta)\bigr),
    \label{eq:hmc_posterior}
\end{equation}
where the potential energy is $
    U(\theta) \triangleq -\log p(\mathcal D|\theta) - \log p_0(\theta)$.
    % \label{eq:potential}
% \end{equation}
HMC augments $\theta$ with an auxiliary momentum variable $r \in \mathbb R^{d_\theta}$, defining the Hamiltonian
\begin{equation}
  H(\theta, r)=
  U(\theta) + V(r)
  = -\log p(\mathcal D | \theta) - \log p_0(\theta)
  + \tfrac{1}{2} r^{\mathsf T} M^{-1} r,
\end{equation}
where $M$ is a symmetric positive-definite mass matrix (often $M=I$). The kinetic energy $V(r) = \tfrac{1}{2} r^{\mathsf T} M^{-1} r$ corresponds to a Gaussian momentum prior $r \sim \mathcal N(0, M)$. The joint distribution is then
\begin{equation}
  p(\theta, r |\mathcal D) \propto \exp\bigl(-H(\theta, r)\bigr).
  \label{eq:joint_distribution}
\end{equation}

At each iteration, HMC samples a fresh momentum $r_0 \sim \mathcal N(0, M)$, and evolves $(\theta, r)$ according to Hamilton’s equations:
\begin{subequations}
\label{eq:hmc_ode}
\begin{align}
  \frac{{\rm d}\theta}{{\rm d}t} &= M^{-1} r, \\
  \frac{{\rm d}r}{{\rm d}t}      &= -\nabla_{\theta} U(\theta).
\end{align}
\end{subequations}

Exact integration conserves $H$ and yields proposals $(\theta',r')$ lying on the Hamiltonian’s energy surface, with $\theta'$ retained as the new sample. In practice, we use the leapfrog discretization, followed by a Metropolis–Hastings correction to offset integration errors~\cite{betancourt2017conceptual}.

\section{Supplementary Numerical Experiments}

In this section, we report the training hyperparameters used for both deterministic and Bayesian PINNs in solving the benchmark PDEs presented in Section~\ref{sec:numerics}.

\label{sec:apd:num}
\begin{table}[t]
  \captionsetup{skip=6pt}
  \centering
  \caption{General training hyperparameter settings across PDEs and models.}
  \label{tab:training_parameter}
  \small
  \begin{tabular}{llcccccc}
    \toprule
    \multirow{2}{*}{\textbf{PDEs}} & \multirow{2}{*}{\textbf{Models}} &
    \multicolumn{3}{c}{\textbf{Loss Weights}} &
    \multirow{2}{*}{\textbf{Epochs}} &
    \multirow{2}{*}{\textbf{Learning rate}} &
    \multirow{2}{*}{\textbf{Seed}} \\
    \cmidrule(lr){3-5}
     & & $\lambda_{\mathrm{pde}}$ & $\lambda_{\mathrm{bc}}$ & $\lambda_{\mathrm{data}}$ & & & \\
    \midrule
    \multirow{5}{*}{2D Allen--Cahn} 
        & GD       & 1.0 & 5.0  & 1.0 & 4500  & 1e-3 & 10 \\
        & LD       & 1.0 & 5.0  & 1.0 & 5000  & 1e-3 & 10 \\
        & Dropout  & 1.0 & 10.0  & 1.0 & 5000  & 1e-3 & 10 \\
        & VI       & 3.0 & 10.0  & 1.0 & 35000 & 1e-3 & 10 \\
        & HMC      & 3.0 & 10.0 & 1.0 & 5000  & 1e-3 & 10 \\
    \midrule
    \multirow{3}{*}{3D Helmholtz} 
        & GD       & 1.0 & 5.0 & 1.0 & 5000  & 1e-3 & 456 \\
        & LD       & 1.0 & 5.0 & 1.0 & 5000  & 1e-3 & 456 \\
        & Dropout  & 1.0 & 10.0 & 1.0 & 5000  & 1e-3 & 456 \\
    \midrule
    \multirow{1}{*}{Ext. 1D Oscillator} 
        & GD       & 1.0 & 10.0  & 3.0 & 20000  & 1e-3 & 95 \\
    \midrule
    \multirow{1}{*}{Ext. 2D Scenario A} 
        & GD       & 1.0 & 5.0  & 1.0 & 4500  & 1e-3 & 259 \\
    \multirow{1}{*}{Ext. 2D Scenario B} 
        & GD       & 1.0 & 5.0  & 1.0 & 4500  & 1e-3 & 711 \\
    \multirow{1}{*}{Ext. 2D Scenario C} 
        & GD       & 1.0 & 5.0  & 1.0 & 4500  & 1e-3 & 345 \\
    \bottomrule
  \end{tabular}
\end{table}

\bibliographystyle{elsarticle-num} 
\bibliography{ft.bib}

% Requires: \usepackage{amsthm,amsmath}
%           (optional) \usepackage{thmtools} for nicer theorem styling

%\begin{thebibliography}{00}

%% \bibitem{label}
%% Text of bibliographic item

%\bibitem{}

%\end{thebibliography}
\end{document}